\pdfoutput=1

\documentclass[11pt]{article}

\usepackage[dvipsnames]{xcolor} 
\usepackage[preprint]{acl}

\usepackage{times}
\usepackage{latexsym}

\usepackage[T1]{fontenc}

\usepackage[utf8]{inputenc}

\usepackage{microtype}

\usepackage{inconsolata}

\usepackage{graphicx}

\usepackage{amsmath}
\usepackage{amssymb}
\usepackage{multirow}
\usepackage{catchfile}

\usepackage[boxed]{algorithm}
\usepackage{varwidth}
\usepackage[noEnd=true,indLines=false]{algpseudocodex}
\usepackage{cleveref}
\makeatletter
\@addtoreset{ALG@line}{algorithm}
\renewcommand{\ALG@beginalgorithmic}{\small}
\algrenewcommand\alglinenumber[1]{\small #1:}
\makeatother

\usepackage[normalem]{ulem}
\usepackage{todonotes}

\usepackage{lipsum}    %
\usepackage{comment}   %
\usepackage{graphicx}  %
\usepackage{pifont}    %

\usepackage[font=small,labelfont=bf]{caption}
\usepackage{float}     %
\usepackage{booktabs}  %
\usepackage{subcaption}  %

\usepackage{listings}

\usepackage{amsthm}  %

\input{./header/definitions.tex}
\title{Better Embeddings with Coupled Adam}

\author{
Felix Stollenwerk\thanks{Corresponding author: \texttt{felix.stollenwerk@ai.se}} \\ AI Sweden
\And  
Tobias Stollenwerk \\ Forschungszentrum Jülich
}

\begin{document}
\maketitle
\begin{abstract}
Despite their remarkable capabilities, LLMs learn word representations that exhibit the undesirable yet poorly understood feature of anisotropy. 
In this paper, we argue that the second moment in Adam is a cause of anisotropic embeddings, and suggest a modified optimizer called Coupled Adam to mitigate the problem.
Our experiments demonstrate that Coupled Adam significantly improves the quality of embeddings, while also leading to better upstream and downstream performance on large enough datasets. 
\end{abstract}

\section{Introduction}
\paragraph{Anisotropic Embeddings}
Large Language Models (LLMs) take a sequence of tokens as input and predict the next token. An embedding matrix is used to map the input tokens to the hidden space of the model, while an unembedding matrix provides the inverse mapping to the output token space. Although the two matrices can in principle be different, it is common practice to apply weight tying \cite{press-wolf-2017-using} and use the transpose of the embedding matrix for unembedding. 
During training, the model learns an embedding vector in hidden space for each token in the vocabulary. However, it is observed that those embedding vectors are clustered in a small subspace away from the origin \cite{gao2019representationdegenerationproblemtraining}. 
This anisotropy limits the semantic usefulness of the embeddings and, in turn, the expressiveness and generalizability of the model.
Multiple attempts have been made to both explain the root cause of the problem and alleviate it (more on this in Sec.~\ref{sec:related_work}). 
In particular, \citet{bis2021tmic} have shown that the problem can be traced back to a mere shift of the mean embedding vector away from the origin. With the mean embedding vector as reference point, the embeddings feature near-perfect isotropy.
However, the role of the employed optimization algorithm has, to the best of our knowledge, not yet been investigated. 

\paragraph{Optimization Algorithms}
Optimization algorithms are an indispensable ingredient in the training of neural networks generally and LLMs in particular.
While SGD is the foundational optimization technique, Adam \cite{adam} is the most widely used optimization techniques for LLMs due to its superior performance and robustness.
While it provides multiple conceptional advantages over SGD, see e.g. \citet{ruder2017overviewgradientdescentoptimization} for a detailed discussion, the one that is particularly striking 
with regard to word embeddings is that Adam is well-suited for sparse data. More concretely, this means that using Adam, the embedding update vectors for rare words are scaled up in comparison to those of more frequent words. This is relevant in the context of LLMs as word frequencies in the training data are typically very skewed and may differ by several orders of magnitude.
Formally, this is captured by the {\em unigram probability distribution} $\widetilde{p} \in [0, 1]^V$, which for a given dataset $d$ and tokenizer $t$ is defined by
\begin{equation}
\widetilde{p}_i \equiv \widetilde{p}_i(d,t) = \frac{n_i}{\sum_j n_j} \; ,
\label{eq:unigram_probability}
\end{equation}
where $i \in \V \equiv \{1, \dots, V\}$
is the vocabulary index and $n_i$ is the total number of occurrences of the $i$-th token in the tokenized dataset. 
A visualization of an example unigram probability distribution can be found in App.~\ref{app:unigram_probability_example}.

\paragraph{Our Contributions}

In this work, we combine the research areas of anisotropic embeddings and optimization algorithms
and provide the following contributions:
\begin{itemize}
\item We show that the Adam optimizer plays a crucial role in causing anisotropic embeddings. 
\item We suggest \textit{Coupled Adam}, an easy-to-implement yet efficient adjustment of the original Adam optimization algorithm, which is specifically designed for embedding parameters in order to alleviate the anisotropy problem. 
\item We demonstrate that our method not only significantly improves the quality of word embeddings, but also has a beneficial effect on upstream and downstream performance for sufficiently large datasets.
\end{itemize}

\section{On the Root Cause of Anisotropic Embeddings}%
\label{sec:theory}

We study the collective shift of the embeddings (that underlies the anisotropy problem), by analyzing their vector updates based on the optimization algorithms SGD and Adam. Weight tying is assumed, but only contributions from the output layer are considered, following \citet{bis2021tmic}. 
Our results apply to all model architectures with a standard language modeling head.

\subsection{Language Modeling Head}

The equations for the standard language modeling head read
\begin{align}
\mathcal{L} &= - \log{(p_t)} \label{eq:forward_loss} \\
    p_t &= \frac{\exp{(l_t)}}{\sum_{j=1}^V \exp{(l_j)}} \label{eq:forward_probability}\\ %
l_i &= e_i \bigcdot h \label{eq:gradient_function_new} \; ,
\end{align}
where $\mathcal{L} \in \mathbb{R}_{\geq 0}$ is the loss for next token prediction, and $p_t \in [0, 1]$ is the predicted probability of the true token $t \in \V$. $l_i \in \mathbb{R}$ and $e_i \in \mathbb{R}^H$ denote the logits and embeddings for each token $i \in \V$, respectively. $h \in \mathbb{R}^H$ is the final hidden state provided by the model for a single token. Note that the operation in Eq.~(\ref{eq:gradient_function_new}) is the dot product of two vectors in $\mathbb{R}^H$.
Backward propagation yields the following gradients with respect to the input vectors $e_i$ and $h$ of Eq.~(\ref{eq:gradient_function_new}):
\begin{align}
g_i :=~ &\frac{\partial \mathcal{L}}{\partial e_i} 
= - \left( \delta_{it} - p_i \right) \cdot h \label{eq:chain_rule_e} 
\end{align}
This result was first reported using a different notation in \citet{bis2021tmic}, and is rederived in App.~\ref{app:chain_rule_e} for the reader's convenience.

\subsection{Vanishing Sum of Embedding Gradients}

Optimization algorithms for neural networks usually update the model parameters iteratively, using an additive update vector that points in direction opposite to the gradient of the loss with respect to the parameters. In the case of embedding vectors, this can be expressed by
\begin{equation}
e_i^{(\tm)} \: = \: e_i^{(\tm-1)} + u_i^{(\tm)}  \; ,
\label{eq:update_general}
\end{equation}
with
\begin{equation}
u_i^{(\tm)} \: \propto \: - g_i^{(\tm)} \; ,
\label{eq:update_vector_definition}
\end{equation}
where $u_i^{(\tm)}$ is the update vector for $e_i^{(\tm)}$ at time step $\tm$.
Eq.~(\ref{eq:chain_rule_e}) implies that the embedding vector $e_t$ of the true token is updated in direction $+h$, while the update vectors $u_i$ for all the other embedding vectors $e_i$ with $i \neq t$ are proportional to $-h$, see Fig.~\ref{fig:gradients_example}. 
\begin{figure}[t]
\centering
\includegraphics[scale=0.5]{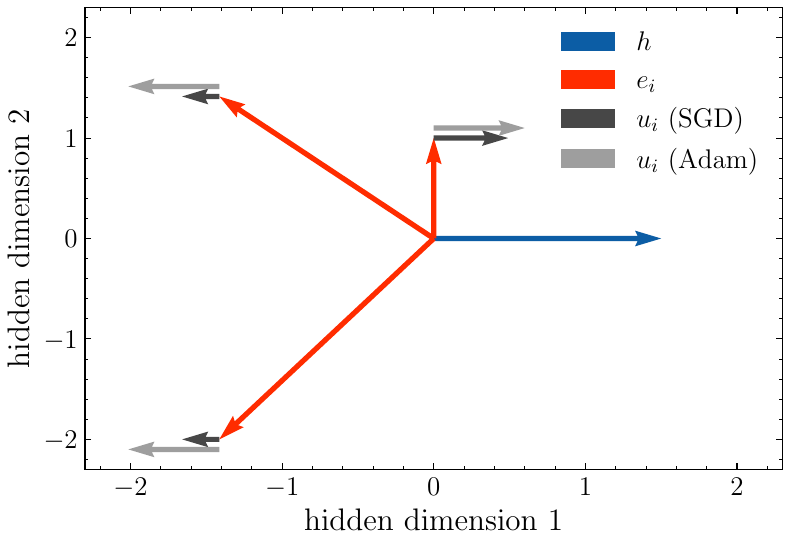}
\caption{Toy example of a hidden state vector $h$ (shown in blue) and three embedding vectors $e_i$ (shown in red) in $H = 2$ dimensions. The gray arrows represent the embedding update vectors, for the SGD (dark) and the Adam (light) optimizer. The update vector of the true token is aligned with $h$, while the others point in the opposite direction, see Eq.~(\ref{eq:chain_rule_e}). Note that the {\em sum of embedding update vectors} vanishes for SGD, while this is not necessarily the case for Adam, cf.~Eqs.~(\ref{eq:vanishing_updates_SGD}) and (\ref{eq:non_vanishing_updates_Adam}).}
\label{fig:gradients_example}
\end{figure}
This circumstance is referred to in the literature as the "common enemy effect" \cite{bis2021tmic}, and regarded as the cause of the representation degeneration problem. 
However, as we will see in the following sections, this explanation is incomplete, as it does not take into account the scaling of the gradients with the predicted probabilities $p_i$, see Eq.~(\ref{eq:chain_rule_e}). The basis for our argumentation is the observation that the {\em sum of embedding gradients vanishes}, as the following simple calculation shows:
\begin{align}
\sum_{i=1}^V g_i^{(\tm)} 
&\stackrel{(\ref{eq:chain_rule_e})}{=} 
- \sum_{i=1}^V  \left( \delta_{it}^{(\tm)} - p_i^{(\tm)} \right) \cdot h^{(\tm)} \nonumber \\
&= - \left( 1 - \sum_{i=1}^V p_i^{(\tm)} \right) \cdot h^{(\tm)} = 0 
\label{eq:optimizer_momentum_conservation}
\end{align}
Next, we will study how Eq.~(\ref{eq:optimizer_momentum_conservation}) translates to the sum $\sum_{i=1}^V u_i^{(\tm)}$
of embedding update vectors, as well as the mean embedding vector
\begin{equation}
\mu^{(\tm)} = \frac{1}{V} \sum_{i=1}^V e_i^{(\tm)}
\label{eq:mu}
\end{equation}
Since the exact definition of the embedding update vector $u_i$, i.e. the proportionality factor in Eq.~(\ref{eq:update_vector_definition}), depends on the optimization algorithm, we discuss SGD and Adam separately.

\subsection{Invariant Mean Embedding with SGD}

We consider the application of the SGD optimization algorithm on the embedding vectors\footnote{Details are given in App.~\ref{app:sgd_algorithm}.}.
At each training step, an embedding vector is simply updated by adding the associated negative gradient $-g_i$, multiplied by a global learning rate $\eta$. Hence, Eq.~(\ref{eq:update_vector_definition}) becomes
\begin{equation}
u_i^{(\tm)} = - \eta \cdot g_i^{(\tm)}
\label{eq:update_vector_definition_SGD}
\end{equation}
Together with Eq.~(\ref{eq:optimizer_momentum_conservation}), this implies that the {\em sum of embedding update vectors vanishes} at any time step $\tm$:
\begin{equation}
\sum_{i=1}^V  u_i^{(\tm)}
\stackrel{(\ref{eq:update_vector_definition_SGD})}{=} - \eta \sum_{i=1}^V g_i^{(\tm)} 
\stackrel{(\ref{eq:optimizer_momentum_conservation})}{=} 0
\label{eq:vanishing_updates_SGD}
\end{equation}
Consequently, the mean embedding vector will stay invariant during the training process:
\begin{equation}
    \mu^{(\tm)} - \mu^{(\tm-1)} 
    \stackrel{(\ref{eq:mu},\ref{eq:update_general})}{=} 
    \frac{1}{V} \sum_{i=1}^V u_i^{(\tm)} 
    \stackrel{(\ref{eq:vanishing_updates_SGD})}{=} 0
    \label{eq:invariant_mean_embedding_SGD}
\end{equation}
This holds even though the different embeddings $e_i$ will be individually updated in different directions with different magnitudes. 
Moreover, all of the above is true also in the case of SGD with momentum,
which follows from linearity and mathematical induction.
Eq.~(\ref{eq:invariant_mean_embedding_SGD}) has far-reaching implications with regard to the anisotropy problem. It entails that the embedding vectors do not collectively shift away from the origin if SGD (with or without momentum) is used. 

\subsection{Shifted Mean Embedding with Adam}

In this section, we analyze the behavior of the mean embedding during optimization with Adam~\cite{adam}, see Algorithm~\ref{alg:algorithm_adam}.
\begin{algorithm}[t]
    \small
    \textbf{Input:}
    $\eta$ (lr), $e_i^{(0)}$ (initial embeddings),
    $\mathcal{L}(e_i)$ (objective), $\beta_1, \beta_2$ (betas), $T$ (number of time steps)
    \\
    \textbf{Initialize:}
    $m_i^{(0)} \leftarrow 0$ (1st moment),
    $v_i^{(0)} \leftarrow 0$ (2nd moment)
    \\
    \textbf{Output}: $e^{(T)}$ (final embeddings)
    \begin{algorithmic}[1]
        \For{$\tm=1 \dots T$}
            \For{$i=1 \dots V$}
                \State $g_i^{(\tm)}$ $\gets$ $\nabla_{e_i} \mathcal{L}^{(\tm)} (e_i^{(\tm-1)})$
                \State $m_i^{(\tm)}$ $\gets$ $\beta_1 m_i^{(\tm-1)} + (1 - \beta_1) g_i^{(\tm)}$%
                \label{alg:line:adam_first_moment_definition}
                \State $v_i^{(\tm)}$ $\gets$ $\beta_2 v_i^{(\tm-1)} + (1-\beta_2) \left(g_i^{(\tm)}\right)^2$%
                \label{alg:line:adam_second_moment_definition}
                \State $\firstmoment$ $\gets$ $m_i^{(\tm)}/\big(1-\beta_1^\tm \big)$%
                \label{alg:line:adam_first_moment_exp_averages}
                \State $\secondmoment$ $\gets$ $v_i^{(\tm)}/\big(1-\beta_2^\tm \big)$%
              \label{alg:line:adam_second_moment_exp_averages}
            \EndFor
            \BeginBox[fill=\colhighlight!10!White, xshift=0.6em, inner xsep=-0.7em]
            \If{\highlight{coupled}}
                \State $\highlight{\secondmomentavg}$ $\gets$ $\highlight{\frac{1}{V} \sum_{i=1}^V \secondmoment}$ %
                \label{alg:line:coupled_adam_second_moment}%
            \EndIf
            \For{$i=1 \dots V$}
                \If{\highlight{coupled}}
                    \State $\highlight{\secondmoment}$ $\gets$ $\highlight{\secondmomentavg}$%
                    \label{alg:line:coupled_adam_second_moment_2}
                \EndIf
            \EndBox
                \State $e_i^{(\tm)}$ $\gets$ $e_i^{(\tm-1)} - \eta \frac{\firstmoment}{\sqrt{\secondmoment} + \epsilon}$%
                \label{alg:line:adam_update}
            \EndFor
        \EndFor
        \State \Return $e^{(T)}$
    \end{algorithmic}
    \caption{Pseudocode for the Adam algorithm and our extension, the \highlight{Coupled Adam algorithm (highlighted)}, applied to the embedding vectors $e_i$. Note that weight decay is not applied.}
    \label{alg:algorithm_adam}
\end{algorithm}

The update vector~Eq.~(\ref{eq:update_vector_definition}) for the Adam algorithm is given by
\begin{equation}
    u_i^{(\tm)} 
    = 
    - \eta^{(\tm)}_i \cdot \firstmoment \label{eq:update_vector_definition_Adam}  \; ,
\end{equation}
where we have introduced an $i$-dependent effective learning rate 
\begin{equation}
    \eta^{(\tm)}_i
    := 
    \frac{\eta}{\sqrt{\secondmoment} + \epsilon} \label{eq:adam_learning_rate} 
\end{equation}
Note that $\firstmoment$ and $\secondmoment$ denote the exponentially averaged first and second moments, respectively, defined according to lines~\ref{alg:line:adam_first_moment_definition}-\ref{alg:line:adam_second_moment_exp_averages}
in Algorithm~\ref{alg:algorithm_adam}.
The $i$-dependent learning rate serves the purpose of individually normalizing the update vectors for different parameters in the Adam optimizer.
However, it also has an unwanted effect specifically on the embedding vectors. While we know from Eq.~\eqref{eq:optimizer_momentum_conservation} and Algorithm~\ref{alg:algorithm_adam} (lines~\ref{alg:line:adam_first_moment_definition},\ref{alg:line:adam_first_moment_exp_averages}) that the {\em unweighted} sum over the first moments vanishes, 
$\sum_{i=1}^V \firstmoment = 0$,
this is not true for the {\em weighted} sum,
\begin{equation}
\sum_{i=1}^V \eta^{(\tm)}_i \firstmoment \neq 0 \; ,
\label{eq:non_vanishing_weighted_sum_of_first_moments_adam}
\end{equation}
unless $\eta^{(\tm)}_i = \eta^{(\tm)}_j$ for all $i, j \in \V$.
Hence, the \textit{sum of embedding update vectors does not vanish} in general,
\begin{equation}
\sum_{i=1}^V  u_i^{(\tm)}
\stackrel{(\ref{eq:update_vector_definition_Adam})}{=} - \sum_{i=1}^V \eta^{(\tm)}_i \cdot \firstmoment 
\stackrel{(\ref{eq:non_vanishing_weighted_sum_of_first_moments_adam})}{\neq} 0
\label{eq:non_vanishing_updates_Adam}
\end{equation}
This, in turn, causes the mean embedding to change during training,
\begin{equation} 
    \mu^{(\tm)} - \mu^{(\tm-1)} 
    \stackrel{(\ref{eq:mu},\ref{eq:update_general})}{=} 
    \frac{1}{V} \sum_{i=1}^V u_i^{(\tm)}
    \stackrel{(\ref{eq:non_vanishing_updates_Adam})}{\neq} 0 \; ,
    \label{eq:mean_embedding_change_adam}
\end{equation}
which is in stark contrast to the case of SGD (cf.~Eq.~\eqref{eq:invariant_mean_embedding_SGD}).
We have thus identified that an $i$-dependency of the second moment $\secondmoment$ of the Adam optimizer leads to the observed collective shift of the embedding vectors away from the origin.
Next, we will show that the second moment indeed depends on $i$. More concretely, we will argue that its expectation value is proportional to the unigram probability\footnote{Note that from here until Eq.~(\ref{eq:optimizer_update_second_moment_avg_canonical}), the time index ($\tau$) is dropped for the sake of readability.} (see Eq.~(\ref{eq:unigram_probability})),
\begin{equation}
    \E \left[ \secondmomentshort \right] \propto \widetilde p_i 
    \label{eq:second_moment_linear_in_unigram_prob}
\end{equation}
In App.~\ref{app:second_moment_theory}, Eq.~(\ref{eq:second_moment_linear_in_unigram_prob}) is derived using minimal assumptions and experimental input.
Here, we restrict ourselves to confirming the relationship in a purely experimental manner. 
$\E \left[ \secondmomentshort \right]$ is estimated directly by measuring $\secondmomentshort$ multiple times during training, using different models. 
We then perform linear fits of $\E \left[ \secondmomentshort \right]$ as a function of $\widetilde p_i$. 
Indeed, the fits yield a high coefficient of determination, on average $R^2 = 0.85(7)$, and a proportionality constant of 
\begin{equation}
A := \frac{\E \left[ \secondmomentshort \right]}{\widetilde p_i} \approx 10^{-4}
\label{eq:second_moment_proportionality_constant}
\end{equation}
Details about the exact procedure and plots showing the data and linear fits can be found in App.~\ref{app:second_moment_empirical}.

\section{Coupled Adam}
\label{sec:lmwap}

In the previous section, we have identified the individual scales of the second moments $v_i$ for different embedding vectors $e_i$ as the root cause of the anisotropy problem. This implies that a solution to the problem is to enforce that the second moments are the same for every $i$.
The question arises whether and how this can be done in the best way, without harming the performance of the model.
To answer this, we note that the normalization of the embedding update vector by the Adam second moment can be split into two parts:
\begin{equation}
\E \left[ \secondmomentshort \right] 
\stackrel{(\ref{eq:second_moment_proportionality_constant})}{=} A \cdot \widetilde p_i
= \frac{A}{V} \cdot \left( \widetilde p_i V \right)
\label{eq:second_moment_factorization}
\end{equation}
The first factor introduces a global scale to all update vectors simultaneously:
\begin{equation}
    \frac{A}{V} \stackrel{(\ref{eq:second_moment_proportionality_constant})}{\approx} \frac{10^{-4}}{5 \cdot 10^{4}} = 2 \cdot 10^{-9} \; ,
\label{eq:second_moment_global_factor}
\end{equation}
where the numbers correspond to our experiments from the previous section with $V \approx 50000$.
The second factor scales the update vectors individually. It is one on average:
\begin{equation}
\frac{1}{V} \sum_{i=1}^{V} \left( \widetilde p_i V \right) = 1
\label{eq:second_moment_individual_factor}
\end{equation}
Our goal is to retain the first, global factor and get rid of the second, individual factor. 
The canonical way to do this is to simply take the average of the second moment over the vocabulary items $i$:
\begin{equation}
\frac{1}{V} \sum_{i=1}^{V} \E \left[ \secondmomentshort \right]
\stackrel{(\ref{eq:second_moment_factorization}, \ref{eq:second_moment_individual_factor})}{=} \frac{A}{V}
\label{eq:optimizer_update_second_moment_avg_canonical}
\end{equation}
In practice, the exponentially averaged second moments $\secondmoment$ as they appear in Eq.~(\ref{eq:adam_learning_rate}) are replaced by their average:
\begin{align}
\secondmomentavg \: &:= \: \frac{1}{V} \sum_{i=1}^V \secondmoment
\label{eq:optimizer_update_second_moment_avg}
\end{align}
We call the resulting algorithm \textit{Coupled Adam}, as it couples the second moments of the embedding vectors via Eq.~(\ref{eq:optimizer_update_second_moment_avg}). 
It is displayed in Algorithm~\ref{alg:algorithm_adam}.
Evidently, with Coupled Adam, the effective learning rate in Eq.~(\ref{eq:adam_learning_rate}) that enters the update vector in Eq.~(\ref{eq:update_vector_definition_Adam}) becomes independent of $i$. Hence, like SGD but unlike standard Adam, the sum of embedding updates vanishes.
However, like standard Adam but unlike SGD, Coupled Adam uses a second moment to normalize the embedding update vectors. 

\section{Experiments}
\label{sec:experiments}

Two types of experiments are conducted to study the impact of coupling the second moments of the embedding update vectors. 
First, a set of small-scale experiments (Sec.~\ref{sec:experiments_S}) with models and datasets of varying sizes up to 1B parameters and 20B tokens, respectively.
Afterwards, we perform a few large-scale experiments (Sec.~\ref{sec:experiments_L}) to verify that the usefulness of our method extrapolates to the realm of large language models with more than 1B parameters trained on at least the corresponding compute-optimal \cite{hoffmann2022trainingcomputeoptimallargelanguage} amount of data.
In order to verify the generalizability of our method, the small- and large-scale experiments involve different datasets, training frameworks and dense transformer model architectures. 
An overview of the model and dataset sizes employed in our experiments is given in App.~\ref{app:experiments_overview}.
For each combination, two models are trained: one using standard Adam and one using Coupled Adam for the embeddings, see Eq.~(\ref{eq:optimizer_update_second_moment_avg}). Both variants use standard Adam for all non-embedding parameters. The various metrics we employ to assess both the general model performance and the quality of the model embeddings will be discussed in Sec.~\ref{sec:experiments_evaluation}.
\subsection{Small-scale Experiments}
\label{sec:experiments_S}
Our small-scale experiments use the OpenWebText Corpus \cite{Gokaslan2019OpenWeb} and the GPT-2 tokenizer \cite{radford2019language}. The model architecture also follows GPT-2, while the hyperparameter setup is taken from GPT-3 \cite{brown2020languagemodelsfewshotlearners}, see App.~\ref{app:hyperparameters} for further details. An implementation based on nanoGPT \cite{Karpathy2022} is used. 
We define a grid $(D, N)$ with dataset sizes 
$D \in \{ 5\B, 10\B, 20\B \}$
and model sizes 
$N \in \{ 125\M, 355\M, 760\M \}$,
and repeat each experiment $S = 3$ times with different seeds in order to estimate uncertainties and assess statistical significance.
\subsection{Large-scale Experiments}
\label{sec:experiments_L}
For our large-scale experiments, we use the SlimPajama dataset \cite{cerebras2023slimpajama} and the GPT-2 tokenizer. A state-of-the-art dense transformer model architecture akin to \cite{touvron2023llama2openfoundation} 
is chosen, including e.g. RoPE embeddings \cite{su2023roformerenhancedtransformerrotary} and the SwiGLU activation function \cite{shazeer2020gluvariantsimprovetransformer}.
Details can be found in App.~\ref{app:hyperparameters}. The experiments are conducted using Modalities \citep{modalities} as the training framework.
We consider two model sizes, 1.3B and 2.6B. In order to cover the two common scenarios of compute-optimal training and overtraining, we conduct two sets of experiments: Firstly, we use near compute-optimal dataset sizes, 26B and 52B tokens, respectively. Secondly, we increase the number of tokens by a factor 4, resulting in 105B and 210B tokens, respectively.
Each large-scale experiment is performed $S = 1$ times.

\subsection{Evaluation}
\label{sec:experiments_evaluation}
Upstream performance is measured in terms of test loss, while downstream performance is evaluated using the Language Model Evaluation Harness \cite{eval-harness} on the following tasks: ARC easy and challenge \cite{clark2018thinksolvedquestionanswering}, HellaSwag \cite{zellers-etal-2019-hellaswag}, LAMBADA \cite{paperno-etal-2016-lambada}, RACE \cite{lai-etal-2017-race}, TruthfulQA \cite{lin-etal-2022-truthfulqa} and WinoGrande \cite{Sakaguchi_LeBras_Bhagavatula_Choi_2020}.
More concretely, the considered metric is the average\footnote{Individual task performance is reported in App.~\ref{app:additional_results_downstream}.} accuracy, which we will denote by $\Acc$.
To assess the quality of the embeddings, we first compute their isotropy, defined as \cite{arora-etal-2016-latent, mu2018allbutthetopsimpleeffectivepostprocessing}
\begin{equation}
\ISO(E) := \frac{\min_{c \in X} Z(c)}{\max_{c \in X} Z(c)} \; ,
\label{eq:isotropy}
\end{equation}
where $E \in \mathbb{R}^{H \times V}$ is the embedding matrix,
$Z(c) = \sum_{i=1}^V \exp(c^T e_i)$ 
is the partition function and $X = \{ c \}$ is the set of eigenvectors $c \in \mathbb{R}^{H}$ of $E E^T \in \mathbb{R}^{H \times H}$.
Secondly, the 2-norm $\munorm$ of the mean embedding, see Eq.~(\ref{eq:mu}), and the average 2-norm of the embeddings
$\overline{\|e_i\|} = \frac{1}{V} \sum_{i=1}^V \|e_i\|$
as well as their ratio 
\begin{equation}
\munormrel := \munorm / \overline{\|e_i\|}
\end{equation}
are determined. 
In addition, we evaluate the models on embedding benchmarks for word similarity and relatedness, to assess how well they represent semantic meaning. Following \citet{bis2021tmic}, we consider the benchmarks SimLex999 \cite{hill-etal-2015-simlex}, MEN \cite{10.5555/2655713.2655714}, WordSim353 \cite{finkelstein} and Stanford Rare Words \cite{luong-etal-2013-better}. Each dataset provides pairs of words labeled with a ground truth score that represents the words' semantic similarity. 
We derive model scores from the cosine similarity of the corresponding embedding vectors, and report the Pearson correlation of the two scores averaged over the datasets, which we denote by $\rcos$.
Finally, some additional important properties of the embedding matrix are investigated.
We study the correlation between the length of an embedding vector and the unigram probability,
\begin{equation}
\rho := 100 \cdot \text{corr} \big( (\|e_i\|)_{i=1}^V, \widetilde p \big) \; ,
\label{eq:rho} %
\end{equation}
to measure how well the former represents the latter.
Furthermore, the condition number $\kappa$, defined as the ratio of the smallest and largest singular values of the embedding matrix, is determined in percent:
\begin{equation}
\kappa := 100 \cdot \frac{\min_i \Sigma_{ii}}{\max_i \Sigma_{ii}}
\label{eq:kappa}
\end{equation}
Here, $E = U \Sigma V^T$ denotes the singular value decomposition of the embedding matrix.

\newcommand{\resultsS}{ \multirow{6}{*}{5B   } & \multirow{2}{*}{   125M   } &   Standard   &  3.14 {\tiny (0)}  &  0.340 {\tiny (2)}  &  0.31 {\tiny (2)}  &  1.10 {\tiny (6)}  &  0.67 {\tiny (3)}  &  15 {\tiny (3)}  &  -54 {\tiny (3)}  &  0.6 {\tiny (1)}\\
                        &                          &   Coupled   & {\bf 3.12 {\tiny (1)}} &  0.339 {\tiny (2)}  & {\bf 0.94 {\tiny (1)}} & {\bf 0.02 {\tiny (0)}} & {\bf 0.01 {\tiny (0)}} & {\bf 55 {\tiny (0)}} & {\bf 87 {\tiny (1)}} & {\bf 4.8 {\tiny (2)}}\\ \cmidrule{2-11}
                        & \multirow{2}{*}{   355M   } &   Standard   &  2.95 {\tiny (0)}  &  0.352 {\tiny (3)}  &  0.44 {\tiny (2)}  &  0.81 {\tiny (2)}  &  0.67 {\tiny (1)}  &  16 {\tiny (2)}  &  -47 {\tiny (0)}  &  0.8 {\tiny (0)}\\
                        &                          &   Coupled   & {\bf 2.93 {\tiny (0)}} &  0.350 {\tiny (4)}  & {\bf 0.98 {\tiny (0)}} & {\bf 0.01 {\tiny (0)}} & {\bf 0.01 {\tiny (0)}} & {\bf 56 {\tiny (1)}} & {\bf 86 {\tiny (0)}} & {\bf 6.8 {\tiny (4)}}\\ \cmidrule{2-11}
                        & \multirow{2}{*}{   760M   } &   Standard   &  2.85 {\tiny (0)}  &  0.360 {\tiny (3)}  &  0.43 {\tiny (1)}  &  0.84 {\tiny (1)}  &  0.63 {\tiny (0)}  &  14 {\tiny (3)}  &  -49 {\tiny (2)}  &  0.7 {\tiny (0)}\\
                        &                          &   Coupled   &  2.86 {\tiny (0)}  &  0.357 {\tiny (3)}  & {\bf 0.97 {\tiny (0)}} & {\bf 0.01 {\tiny (0)}} & {\bf 0.01 {\tiny (0)}} & {\bf 55 {\tiny (1)}} & {\bf 85 {\tiny (1)}} & {\bf 6.9 {\tiny (3)}}\\ \midrule
 \multirow{6}{*}{10B   } & \multirow{2}{*}{   125M   } &   Standard   &  3.07 {\tiny (0)}  &  0.343 {\tiny (3)}  &  0.21 {\tiny (3)}  &  1.58 {\tiny (5)}  &  0.75 {\tiny (0)}  &  9 {\tiny (2)}  &  -64 {\tiny (5)}  &  0.4 {\tiny (0)}\\
                        &                          &   Coupled   & {\bf 3.03 {\tiny (0)}} &  0.343 {\tiny (1)}  & {\bf 0.91 {\tiny (1)}} & {\bf 0.05 {\tiny (0)}} & {\bf 0.02 {\tiny (0)}} & {\bf 57 {\tiny (2)}} & {\bf 82 {\tiny (0)}} & {\bf 3.6 {\tiny (8)}}\\ \cmidrule{2-11}
                        & \multirow{2}{*}{   355M   } &   Standard   &  2.86 {\tiny (0)}  &  0.359 {\tiny (2)}  &  0.35 {\tiny (2)}  &  1.01 {\tiny (4)}  &  0.74 {\tiny (2)}  &  10 {\tiny (3)}  &  -55 {\tiny (3)}  &  0.5 {\tiny (0)}\\
                        &                          &   Coupled   & {\bf 2.83 {\tiny (0)}} & {\bf 0.365 {\tiny (2)}} & {\bf 0.96 {\tiny (0)}} & {\bf 0.02 {\tiny (0)}} & {\bf 0.01 {\tiny (0)}} & {\bf 57 {\tiny (1)}} & {\bf 83 {\tiny (1)}} & {\bf 5.3 {\tiny (2)}}\\ \cmidrule{2-11}
                        & \multirow{2}{*}{   760M   } &   Standard   &  2.75 {\tiny (0)}  &  0.375 {\tiny (2)}  &  0.38 {\tiny (3)}  &  0.97 {\tiny (4)}  &  0.66 {\tiny (2)}  &  11 {\tiny (2)}  &  -56 {\tiny (1)}  &  0.5 {\tiny (0)}\\
                        &                          &   Coupled   &  2.74 {\tiny (1)}  &  0.372 {\tiny (3)}  & {\bf 0.96 {\tiny (0)}} & {\bf 0.02 {\tiny (0)}} & {\bf 0.01 {\tiny (0)}} & {\bf 57 {\tiny (0)}} & {\bf 84 {\tiny (0)}} & {\bf 6.2 {\tiny (0)}}\\ \midrule
 \multirow{6}{*}{20B   } & \multirow{2}{*}{   125M   } &   Standard   &  3.03 {\tiny (0)}  &  0.346 {\tiny (1)}  &  0.10 {\tiny (3)}  &  2.14 {\tiny (7)}  &  0.82 {\tiny (2)}  &  5 {\tiny (1)}  &  -66 {\tiny (2)}  &  0.3 {\tiny (0)}\\
                        &                          &   Coupled   & {\bf 2.97 {\tiny (0)}} & {\bf 0.350 {\tiny (1)}} & {\bf 0.83 {\tiny (1)}} & {\bf 0.11 {\tiny (0)}} & {\bf 0.03 {\tiny (0)}} & {\bf 57 {\tiny (2)}} & {\bf 77 {\tiny (0)}} & {\bf 1.7 {\tiny (5)}}\\ \cmidrule{2-11}
                        & \multirow{2}{*}{   355M   } &   Standard   &  2.79 {\tiny (0)}  &  0.366 {\tiny (4)}  &  0.25 {\tiny (2)}  &  1.32 {\tiny (8)}  &  0.82 {\tiny (2)}  &  5 {\tiny (2)}  &  -65 {\tiny (4)}  &  0.3 {\tiny (0)}\\
                        &                          &   Coupled   & {\bf 2.75 {\tiny (0)}} &  0.372 {\tiny (6)}  & {\bf 0.95 {\tiny (2)}} & {\bf 0.04 {\tiny (0)}} & {\bf 0.02 {\tiny (0)}} & {\bf 57 {\tiny (1)}} & {\bf 78 {\tiny (0)}} & {\bf 4.1 {\tiny (3)}}\\ \cmidrule{2-11}
                        & \multirow{2}{*}{   760M   } &   Standard   &  2.68 {\tiny (1)}  &  0.385 {\tiny (3)}  &  0.28 {\tiny (2)}  &  1.21 {\tiny (8)}  &  0.73 {\tiny (3)}  &  3 {\tiny (4)}  &  -64 {\tiny (2)}  &  0.3 {\tiny (0)}\\
                        &                          &   Coupled   & {\bf 2.65 {\tiny (0)}} & {\bf 0.392 {\tiny (2)}} & {\bf 0.94 {\tiny (2)}} & {\bf 0.03 {\tiny (0)}} & {\bf 0.01 {\tiny (0)}} & {\bf 58 {\tiny (0)}} & {\bf 81 {\tiny (0)}} & {\bf 4.4 {\tiny (2)}}\\ 
}
\newcommand{\resultsL}{ \multirow{2}{*}{26B  } & \multirow{2}{*}{  1.3B  } &  Standard  & {\bf 2.433} & {\bf 0.402} &  0.50  &  0.67  &  0.45  &  53  &  -41  &  1.2\\
                        &                          &  Coupled  &  2.448  &  0.396  & {\bf 0.96} & {\bf 0.05} & {\bf 0.03} & {\bf 66} & {\bf 74} & {\bf 4.5}\\ \midrule
 \multirow{2}{*}{52B  } & \multirow{2}{*}{  2.6B  } &  Standard  & {\bf 2.257} & {\bf 0.451} &  0.51  &  0.68  &  0.40  &  55  &  -44  &  1.0\\
                        &                          &  Coupled  &  2.273  &  0.441  & {\bf 0.86} & {\bf 0.10} & {\bf 0.06} & {\bf 66} & {\bf 74} & {\bf 3.0}\\ \midrule
 \multirow{2}{*}{105B  } & \multirow{2}{*}{  1.3B  } &  Standard  &  2.278  &  0.446  &  0.40  &  0.87  &  0.43  &  52  &  -52  &  0.6\\
                        &                          &  Coupled  & {\bf 2.277} & {\bf 0.447} & {\bf 0.82} & {\bf 0.23} & {\bf 0.09} & {\bf 67} & {\bf 71} & {\bf 2.2}\\ \midrule
 \multirow{2}{*}{210B  } & \multirow{2}{*}{  2.6B  } &  Standard  &  2.131  &  0.490  &  0.34  &  0.96  &  0.41  &  54  &  -52  &  0.5\\
                        &                          &  Coupled  & {\bf 2.129} & {\bf 0.492} & {\bf 0.65} & {\bf 0.49} & {\bf 0.14} & {\bf 67} & {\bf 69} & {\bf 1.5}\\ 
}

\begin{table*}[!ht]
\centering
\scriptsize
\begin{tabular}{ccc|rrrrrrrr}
\toprule
$D$ & $N$ & Adam & $\Loss$ ($\downarrow$) & $\Acc$ ($\uparrow$) & $\ISO$ ($\uparrow$) & $\munorm$ ($\downarrow$) & $\munormrel$ ($\downarrow$) & $\rcos$ ($\uparrow$) & $\rho$ ($\uparrow$) & $\kappa$ ($\uparrow$) \\ 
\midrule
 \multirow{6}{*}{5B   } & \multirow{2}{*}{   125M   } &   Standard   &  3.14 {\tiny (0)}  &  0.340 {\tiny (2)}  &  0.31 {\tiny (2)}  &  1.10 {\tiny (6)}  &  0.67 {\tiny (3)}  &  15 {\tiny (3)}  &  -54 {\tiny (3)}  &  0.6 {\tiny (1)}\\
                        &                          &   Coupled   & {\bf 3.12 {\tiny (1)}} &  0.339 {\tiny (2)}  & {\bf 0.94 {\tiny (1)}} & {\bf 0.02 {\tiny (0)}} & {\bf 0.01 {\tiny (0)}} & {\bf 55 {\tiny (0)}} & {\bf 87 {\tiny (1)}} & {\bf 4.8 {\tiny (2)}}\\ \cmidrule{2-11}
                        & \multirow{2}{*}{   355M   } &   Standard   &  2.95 {\tiny (0)}  &  0.352 {\tiny (3)}  &  0.44 {\tiny (2)}  &  0.81 {\tiny (2)}  &  0.67 {\tiny (1)}  &  16 {\tiny (2)}  &  -47 {\tiny (0)}  &  0.8 {\tiny (0)}\\
                        &                          &   Coupled   & {\bf 2.93 {\tiny (0)}} &  0.350 {\tiny (4)}  & {\bf 0.98 {\tiny (0)}} & {\bf 0.01 {\tiny (0)}} & {\bf 0.01 {\tiny (0)}} & {\bf 56 {\tiny (1)}} & {\bf 86 {\tiny (0)}} & {\bf 6.8 {\tiny (4)}}\\ \cmidrule{2-11}
                        & \multirow{2}{*}{   760M   } &   Standard   &  2.85 {\tiny (0)}  &  0.360 {\tiny (3)}  &  0.43 {\tiny (1)}  &  0.84 {\tiny (1)}  &  0.63 {\tiny (0)}  &  14 {\tiny (3)}  &  -49 {\tiny (2)}  &  0.7 {\tiny (0)}\\
                        &                          &   Coupled   &  2.86 {\tiny (0)}  &  0.357 {\tiny (3)}  & {\bf 0.97 {\tiny (0)}} & {\bf 0.01 {\tiny (0)}} & {\bf 0.01 {\tiny (0)}} & {\bf 55 {\tiny (1)}} & {\bf 85 {\tiny (1)}} & {\bf 6.9 {\tiny (3)}}\\ \midrule
 \multirow{6}{*}{10B   } & \multirow{2}{*}{   125M   } &   Standard   &  3.07 {\tiny (0)}  &  0.343 {\tiny (3)}  &  0.21 {\tiny (3)}  &  1.58 {\tiny (5)}  &  0.75 {\tiny (0)}  &  9 {\tiny (2)}  &  -64 {\tiny (5)}  &  0.4 {\tiny (0)}\\
                        &                          &   Coupled   & {\bf 3.03 {\tiny (0)}} &  0.343 {\tiny (1)}  & {\bf 0.91 {\tiny (1)}} & {\bf 0.05 {\tiny (0)}} & {\bf 0.02 {\tiny (0)}} & {\bf 57 {\tiny (2)}} & {\bf 82 {\tiny (0)}} & {\bf 3.6 {\tiny (8)}}\\ \cmidrule{2-11}
                        & \multirow{2}{*}{   355M   } &   Standard   &  2.86 {\tiny (0)}  &  0.359 {\tiny (2)}  &  0.35 {\tiny (2)}  &  1.01 {\tiny (4)}  &  0.74 {\tiny (2)}  &  10 {\tiny (3)}  &  -55 {\tiny (3)}  &  0.5 {\tiny (0)}\\
                        &                          &   Coupled   & {\bf 2.83 {\tiny (0)}} & {\bf 0.365 {\tiny (2)}} & {\bf 0.96 {\tiny (0)}} & {\bf 0.02 {\tiny (0)}} & {\bf 0.01 {\tiny (0)}} & {\bf 57 {\tiny (1)}} & {\bf 83 {\tiny (1)}} & {\bf 5.3 {\tiny (2)}}\\ \cmidrule{2-11}
                        & \multirow{2}{*}{   760M   } &   Standard   &  2.75 {\tiny (0)}  &  0.375 {\tiny (2)}  &  0.38 {\tiny (3)}  &  0.97 {\tiny (4)}  &  0.66 {\tiny (2)}  &  11 {\tiny (2)}  &  -56 {\tiny (1)}  &  0.5 {\tiny (0)}\\
                        &                          &   Coupled   &  2.74 {\tiny (1)}  &  0.372 {\tiny (3)}  & {\bf 0.96 {\tiny (0)}} & {\bf 0.02 {\tiny (0)}} & {\bf 0.01 {\tiny (0)}} & {\bf 57 {\tiny (0)}} & {\bf 84 {\tiny (0)}} & {\bf 6.2 {\tiny (0)}}\\ \midrule
 \multirow{6}{*}{20B   } & \multirow{2}{*}{   125M   } &   Standard   &  3.03 {\tiny (0)}  &  0.346 {\tiny (1)}  &  0.10 {\tiny (3)}  &  2.14 {\tiny (7)}  &  0.82 {\tiny (2)}  &  5 {\tiny (1)}  &  -66 {\tiny (2)}  &  0.3 {\tiny (0)}\\
                        &                          &   Coupled   & {\bf 2.97 {\tiny (0)}} & {\bf 0.350 {\tiny (1)}} & {\bf 0.83 {\tiny (1)}} & {\bf 0.11 {\tiny (0)}} & {\bf 0.03 {\tiny (0)}} & {\bf 57 {\tiny (2)}} & {\bf 77 {\tiny (0)}} & {\bf 1.7 {\tiny (5)}}\\ \cmidrule{2-11}
                        & \multirow{2}{*}{   355M   } &   Standard   &  2.79 {\tiny (0)}  &  0.366 {\tiny (4)}  &  0.25 {\tiny (2)}  &  1.32 {\tiny (8)}  &  0.82 {\tiny (2)}  &  5 {\tiny (2)}  &  -65 {\tiny (4)}  &  0.3 {\tiny (0)}\\
                        &                          &   Coupled   & {\bf 2.75 {\tiny (0)}} &  0.372 {\tiny (6)}  & {\bf 0.95 {\tiny (2)}} & {\bf 0.04 {\tiny (0)}} & {\bf 0.02 {\tiny (0)}} & {\bf 57 {\tiny (1)}} & {\bf 78 {\tiny (0)}} & {\bf 4.1 {\tiny (3)}}\\ \cmidrule{2-11}
                        & \multirow{2}{*}{   760M   } &   Standard   &  2.68 {\tiny (1)}  &  0.385 {\tiny (3)}  &  0.28 {\tiny (2)}  &  1.21 {\tiny (8)}  &  0.73 {\tiny (3)}  &  3 {\tiny (4)}  &  -64 {\tiny (2)}  &  0.3 {\tiny (0)}\\
                        &                          &   Coupled   & {\bf 2.65 {\tiny (0)}} & {\bf 0.392 {\tiny (2)}} & {\bf 0.94 {\tiny (2)}} & {\bf 0.03 {\tiny (0)}} & {\bf 0.01 {\tiny (0)}} & {\bf 58 {\tiny (0)}} & {\bf 81 {\tiny (0)}} & {\bf 4.4 {\tiny (2)}}\\ 

\bottomrule 
\end{tabular}
\caption{Results of our small-scale experiments. $D$ and $N$ denote the dataset and model size, respectively. $\Loss$ is the test loss, and the column $\Acc$ represents the accuracy averaged over the downstream tasks listed in Sec.~\ref{sec:experiments_evaluation}. The other evaluation metrics are defined in the same section, see Eqs.~(\ref{eq:isotropy})-(\ref{eq:kappa}). The arrow in parentheses indicates whether a higher or lower value is desirable. Every training was conducted $S=3$ times with different seeds, and the numbers represent the (rounded) averages and standard deviations in the following shorthand notation format: $0.123${\scriptsize~$(4)$} $\equiv 0.123 \pm 0.004$. For each combination $(D, N)$ and each metric, the respective better value is highlighted in bold if the (unrounded) difference is significant according to Student's t-test with a one-sided confidence level of $\alpha = 95\%$ (see App.~\ref{app:error} for details). Plots for $\Loss$ and $\Acc$ are shown in Fig.~\ref{fig:results_S}.}
\label{tab:results_S}
\end{table*}
\begin{figure*}[!ht]
    \centering
    \includegraphics[scale=0.5]{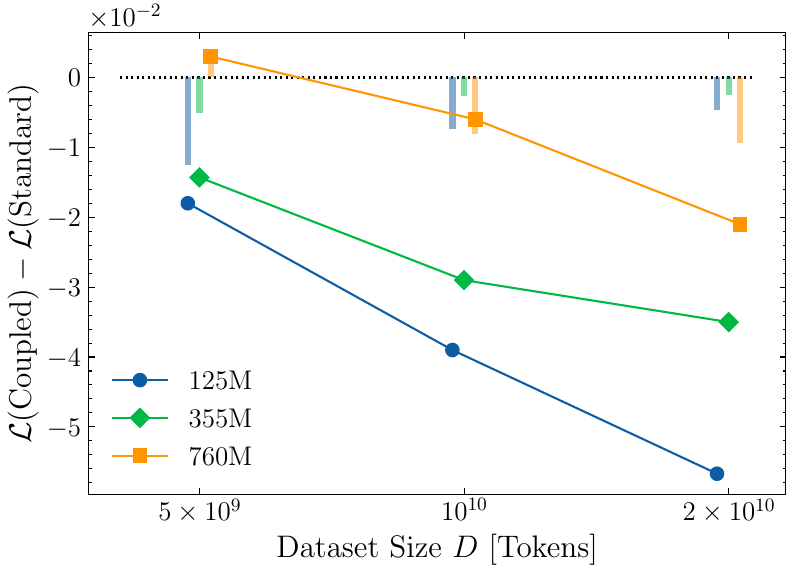} \qquad
    \includegraphics[scale=0.5]{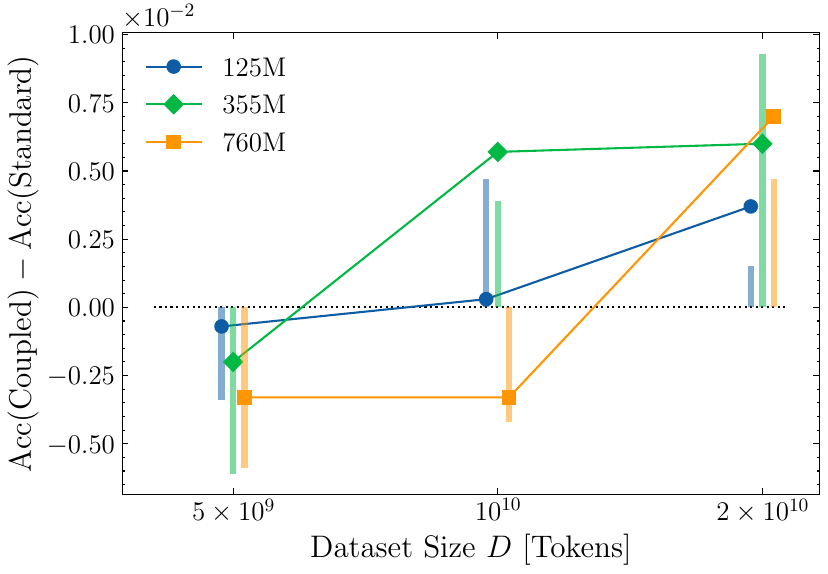}
    \caption{Difference in loss (left) and average downstream task accuracy (right) between Coupled Adam and standard Adam, for the different dataset sizes $D$ (horizontal axis) and model sizes $N$ (colors) of the small-scale experiments. The vertical bars indicate the one-sided $95\%$ confidence interval for the difference to be significant. In order to avoid overlaps, the data points for $N=125\M$ and $N=760\M$ have been slightly shifted to the left and right, respectively.}
    \label{fig:results_S}
\end{figure*}

\section{Results}
\label{sec:results}

\subsection{Small-scale Experiments}
\label{sec:results_S}

The results of the small-scale experiments (Sec.~\ref{sec:experiments_S}) are listed in Tab.~\ref{tab:results_S} and illustrated in Fig.~\ref{fig:results_S}.
We find that both upstream and downstream performance are better with Coupled Adam if the dataset size is sufficiently large. In fact, the improvement appears to increase monotonically with the dataset size $D$. 
In addition, the embedding-specific metrics benefit greatly from Coupled Adam. In particular, the isotropy reaches values above $0.90$ (with a single exception), while $\rcos$ and $\kappa$ are hugely improved as well. The mean embedding is evidently close to the origin. Finally, Coupled Adam leads to a significantly stronger (positive) correlation $\rho$ between the length of an embedding vector and its associated unigram probability.

\subsection{Large-scale Experiments}
\label{sec:results_L}

The results of the large-scale experiments (Sec.~\ref{sec:experiments_L}) are shown in Tab.~\ref{tab:results_L}.
\begin{table*}[ht]
\centering
\scriptsize
\begin{tabular}{lll|rrrrrrrr}
\toprule
$D$ & $N$ & Adam & $\Loss$ ($\downarrow$) & $\Acc$ ($\uparrow$) & $\ISO$ ($\uparrow$) & $\munorm$ ($\downarrow$) & $\munormrel$ ($\downarrow$) & $\rcos$ ($\uparrow$) & $\rho$ ($\uparrow$) & $\kappa$ ($\uparrow$) \\ 
\midrule
 \multirow{2}{*}{26B  } & \multirow{2}{*}{  1.3B  } &  Standard  & {\bf 2.433} & {\bf 0.402} &  0.50  &  0.67  &  0.45  &  53  &  -41  &  1.2\\
                        &                          &  Coupled  &  2.448  &  0.396  & {\bf 0.96} & {\bf 0.05} & {\bf 0.03} & {\bf 66} & {\bf 74} & {\bf 4.5}\\ \midrule
 \multirow{2}{*}{52B  } & \multirow{2}{*}{  2.6B  } &  Standard  & {\bf 2.257} & {\bf 0.451} &  0.51  &  0.68  &  0.40  &  55  &  -44  &  1.0\\
                        &                          &  Coupled  &  2.273  &  0.441  & {\bf 0.86} & {\bf 0.10} & {\bf 0.06} & {\bf 66} & {\bf 74} & {\bf 3.0}\\ \midrule
 \multirow{2}{*}{105B  } & \multirow{2}{*}{  1.3B  } &  Standard  &  2.278  &  0.446  &  0.40  &  0.87  &  0.43  &  52  &  -52  &  0.6\\
                        &                          &  Coupled  & {\bf 2.277} & {\bf 0.447} & {\bf 0.82} & {\bf 0.23} & {\bf 0.09} & {\bf 67} & {\bf 71} & {\bf 2.2}\\ \midrule
 \multirow{2}{*}{210B  } & \multirow{2}{*}{  2.6B  } &  Standard  &  2.131  &  0.490  &  0.34  &  0.96  &  0.41  &  54  &  -52  &  0.5\\
                        &                          &  Coupled  & {\bf 2.129} & {\bf 0.492} & {\bf 0.65} & {\bf 0.49} & {\bf 0.14} & {\bf 67} & {\bf 69} & {\bf 1.5}\\ 

\bottomrule 
\end{tabular}
\caption{Results of our large-scale experiments. See the caption of Tab.~\ref{tab:results_S} for an explanation of the column names. For each combination $(D, N)$ and each metric, the respective better value is highlighted in bold.}
\label{tab:results_L}
\end{table*}
We observe very similar patterns as for the small-scale experiments. Although upstream and downstream performance are worse with Coupled Adam for compute-optimal dataset sizes, they are better if 4 times larger datasets are used. Note that for the small-scale experiments, the upstream and downstream performance were found to be better already for compute-optimal dataset sizes. We attribute this to the fact that the batch size for the large-scale experiments is five times larger (cf.~App.~\ref{app:hyperparameters}), which results in fewer optimization steps for the same dataset size.
Regarding the embedding-specific metrics, we again find significant and consistent improvements throughout all experiments. 
However, we do observe a certain shift of the mean embedding vector away from the origin, even if Coupled Adam is used. The shift becomes more pronounced as the model and dataset sizes increase, and is also reflected in a reduced isotropy.
As we shall see in the following section, it comes along with optimal model performance though. 
An obvious hypothesis in light of our analysis in Sec.~\ref{sec:theory} is that the residual shift of the mean embeddings is due to weight tying. This is supported by the results of \citet{machina-mercer-2024-anisotropy}, who find improved isotropy for models without weight tying. We leave it for future work to verify the hypothesis.

\newcommand{\resultsAblationsScaleSmall}{  -5   &  2.99 {\tiny (0)}  &  0.349 {\tiny (2)}  &  0.97 {\tiny (0)}  & {\bf 0.01 {\tiny (0)}} &  0.01 {\tiny (0)}  &  55 {\tiny (1)}  &  76 {\tiny (3)}  & {\bf 6.0 {\tiny (6)}}\\
  -4   &  2.99 {\tiny (0)}  &  0.348 {\tiny (5)}  &  0.97 {\tiny (1)}  &  0.02 {\tiny (0)}  &  0.01 {\tiny (0)}  &  55 {\tiny (1)}  &  78 {\tiny (3)}  &  4.4 {\tiny (1)}\\
  -3   &  2.98 {\tiny (0)}  &  0.352 {\tiny (5)}  &  0.95 {\tiny (1)}  &  0.03 {\tiny (0)}  &  0.02 {\tiny (0)}  &  57 {\tiny (2)}  &  78 {\tiny (2)}  &  2.7 {\tiny (5)}\\
  -2   &  2.98 {\tiny (0)}  &  0.352 {\tiny (1)}  &  0.94 {\tiny (2)}  &  0.04 {\tiny (0)}  &  0.02 {\tiny (0)}  &  57 {\tiny (1)}  &  78 {\tiny (2)}  &  2.3 {\tiny (4)}\\
  -1   &  2.98 {\tiny (0)}  &  0.348 {\tiny (3)}  &  0.87 {\tiny (2)}  &  0.07 {\tiny (0)}  &  0.02 {\tiny (0)}  &  57 {\tiny (1)}  &  79 {\tiny (2)}  &  2.1 {\tiny (3)}\\ \cmidrule{1-9}
   0   &  2.97 {\tiny (0)}  &  0.350 {\tiny (1)}  &  0.83 {\tiny (1)}  &  0.11 {\tiny (0)}  &  0.03 {\tiny (0)}  &  57 {\tiny (2)}  &  77 {\tiny (0)}  &  1.7 {\tiny (5)}\\ \cmidrule{1-9}
   1   &  2.97 {\tiny (0)}  &  0.351 {\tiny (4)}  &  0.66 {\tiny (6)}  &  0.20 {\tiny (1)}  &  0.03 {\tiny (0)}  &  56 {\tiny (0)}  &  78 {\tiny (2)}  &  1.6 {\tiny (5)}\\
   2   &  2.98 {\tiny (0)}  &  0.353 {\tiny (2)}  &  0.47 {\tiny (6)}  &  0.34 {\tiny (1)}  &  0.04 {\tiny (0)}  &  58 {\tiny (1)}  &  78 {\tiny (2)}  &  1.6 {\tiny (9)}\\
   3   &  2.97 {\tiny (0)}  &  0.352 {\tiny (0)}  &  0.27 {\tiny (4)}  &  0.54 {\tiny (2)}  &  0.05 {\tiny (0)}  &  58 {\tiny (2)}  &  78 {\tiny (1)}  &  2.0 {\tiny (7)}\\
   4   &  2.97 {\tiny (1)}  &  0.352 {\tiny (1)}  &  0.09 {\tiny (0)}  &  0.83 {\tiny (2)}  &  0.05 {\tiny (0)}  &  58 {\tiny (1)}  &  77 {\tiny (1)}  &  2.3 {\tiny (2)}\\
   5   &  2.98 {\tiny (0)}  &  0.349 {\tiny (4)}  &  0.01 {\tiny (1)}  &  1.32 {\tiny (2)}  &  0.06 {\tiny (0)}  &  57 {\tiny (1)}  &  75 {\tiny (1)}  &  2.1 {\tiny (6)}\\
}
\newcommand{\resultsAblationsSGDAll}{ \multirow{2}{*}{5B   } & \multirow{2}{*}{   125M   } &   SGD (300)   &  3.17 {\tiny (0)}  &  0.333 {\tiny (3)}  & {\bf 0.99 {\tiny (0)}} & {\bf 0.00 {\tiny (0)}} & {\bf 0.01 {\tiny (0)}} &  45 {\tiny (1)}  &  71 {\tiny (1)}  & {\bf 15.5 {\tiny (4)}}\\
                        &                          &   Coupled Adam   & {\bf 3.12 {\tiny (1)}} & {\bf 0.339 {\tiny (2)}} &  0.94 {\tiny (1)}  &  0.02 {\tiny (0)}  &  0.01 {\tiny (0)}  & {\bf 55 {\tiny (0)}} & {\bf 87 {\tiny (1)}} &  4.8 {\tiny (2)}\\ \midrule
 \multirow{2}{*}{10B   } & \multirow{2}{*}{   125M   } &   SGD (300)   &  3.07 {\tiny (1)}  &  0.341 {\tiny (4)}  & {\bf 0.99 {\tiny (0)}} & {\bf 0.01 {\tiny (0)}} & {\bf 0.01 {\tiny (0)}} &  49 {\tiny (1)}  &  71 {\tiny (4)}  & {\bf 12.2 {\tiny (2)}}\\
                        &                          &   Coupled Adam   & {\bf 3.03 {\tiny (0)}} &  0.343 {\tiny (1)}  &  0.91 {\tiny (1)}  &  0.05 {\tiny (0)}  &  0.02 {\tiny (0)}  & {\bf 57 {\tiny (2)}} & {\bf 82 {\tiny (0)}} &  3.6 {\tiny (8)}\\ \midrule
 \multirow{2}{*}{20B   } & \multirow{2}{*}{   125M   } &   SGD (400)   &  3.00 {\tiny (0)}  &  0.346 {\tiny (5)}  & {\bf 0.98 {\tiny (1)}} & {\bf 0.01 {\tiny (0)}} & {\bf 0.02 {\tiny (0)}} &  54 {\tiny (0)}  &  76 {\tiny (5)}  & {\bf 7.4 {\tiny (1.1)}}\\
                        &                          &   Coupled Adam   & {\bf 2.97 {\tiny (0)}} &  0.350 {\tiny (1)}  &  0.83 {\tiny (1)}  &  0.11 {\tiny (0)}  &  0.03 {\tiny (0)}  &  57 {\tiny (2)}  &  77 {\tiny (0)}  &  1.7 {\tiny (5)}\\ 
}

\section{Ablations}
\label{sec:ablations}

We perform some additional experiments to shed further light on how Coupled Adam works. A model size of $N = 125\M$ and the dataset sizes $D \in \{ 5\B, 10\B, 20\B \}$ from the small-scale experiments (Sec.~\ref{sec:experiments_S}) are used, and each experiment is repeated $S = 3$ times with different seeds. 

\subsection{Scaled Coupled Adam}
\label{sec:ablation_different_learning_rate}

While coupling the second moment of the embedding gradients using the average in Eq.~(\ref{eq:optimizer_update_second_moment_avg}) is the canonical choice, one could also use a multiple of the average. We conduct additional experiments where the coupled second moment is scaled by powers of $2$:
\begin{equation}
\secondmomentavg \: \to \: 2^{-n} \cdot \secondmomentavg  \; ,
\label{eq:optimizer_update_second_moment_avg_scaled}
\end{equation}
with scaling exponents
$n \in \{ z \in \mathbb{Z} ~| -5 \leq z \leq 5 \}$.
Note that using a scaling exponent $n \neq 0$ is equivalent to using a different effective learning rate for the embeddings than for all the other parameters, via Eqs.~(\ref{eq:optimizer_update_second_moment_avg}) and (\ref{eq:adam_learning_rate}). In particular, a smaller scaling exponent $n$ corresponds to a smaller effective learning rate and vice versa. 
The results for $D=20\B$ are shown in Tab.~\ref{tab:results_ablations_scale}, 
and the dependency of the loss on the scaling exponent $n$ for that very dataset size is visualized in Fig.~\ref{fig:ablation_scale}.
\begin{figure*}[t]
\begin{minipage}{\textwidth}
  \begin{minipage}[b]{0.63\textwidth}
    \centering
    \scriptsize
    \begin{tabular}{c|rrrrrrrr}
    \toprule
    $n$ & $\Loss$ ($\downarrow$) & $\Acc$ ($\uparrow$) & $\ISO$ ($\uparrow$) & $\munorm$ ($\downarrow$) & $\munormrel$ ($\downarrow$) & $\rcos$ ($\uparrow$) & $\rho$ ($\uparrow$) & $\kappa$ ($\uparrow$) \\ 
    \midrule
    \input{tables/results_ablations_scale_small.tex}
    \bottomrule 
    \end{tabular}
    \captionof{table}{Results of our experiments with Scaled Coupled Adam, for $N=125\M$ and $D=20\B$. Values are highlighted in bold if they are significantly better than \textit{all} the other values in the same column, see the caption of Tab.~\ref{tab:results_S} for more details.}
    \label{tab:results_ablations_scale}
  \end{minipage}
  \hfill
  \begin{minipage}[b]{0.35\textwidth}
    \centering
    \includegraphics[scale=0.40]{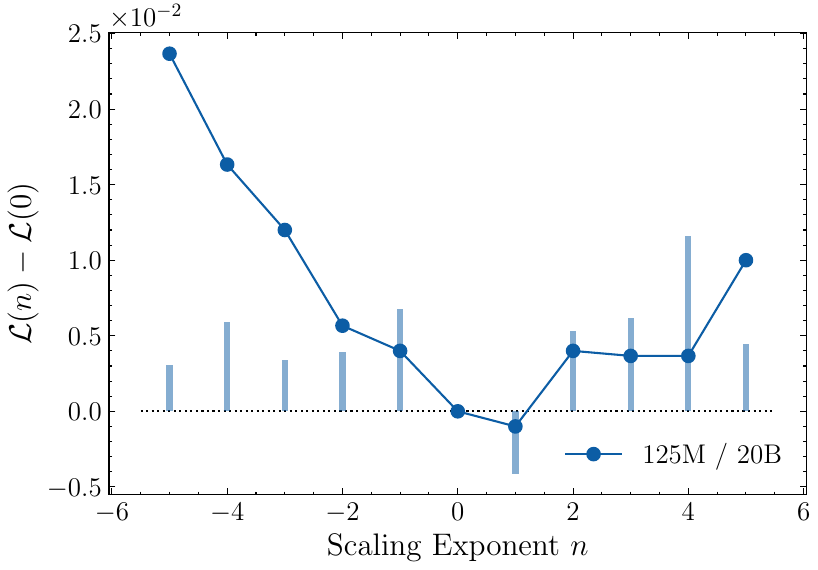}
    \captionof{figure}{Dependency of the loss on the scaling exponent $n$, see Eq.~(\ref{eq:optimizer_update_second_moment_avg_scaled}), for $N=125\M$ and $D=20\B$. The plot shows the difference to the loss obtained for $n = 0$.}
    \label{fig:ablation_scale}
  \end{minipage}
\end{minipage}
\end{figure*}
Results for other dataset sizes and plots for the other evaluation metrics can be found in App.~\ref{app:additional_results_scale}.
Our data shows that the loss reaches a minimum close to $n=0$, with a rather weak dependence on the scaling exponent in its vicinity. Nevertheless, for the smallest and largest scaling exponents studied, we find that the loss gets significantly worse.
Regarding downstream performance, we see indications of a similar pattern, although the statistical uncertainties are too large to draw definite conclusions.  
The semantic usefulness of the embedding vectors as measured by $\rcos$ seems to suffer from a scaling exponent $n < 0$. For the isotropy and the mean embedding, we observe the opposite behavior. They benefit from a smaller scaling exponent $n$ and the associated smaller embedding updates, with the effect being more pronounced the larger the training dataset size $D$.  
However, this also negatively affects the model performance. 
Hence, we conclude that, at least within the range of our experiments, the optimal setting is to have the same learning rate for the embedding parameters as for all the other model parameters, as implied by $n=0$ and Eq.~(\ref{eq:optimizer_update_second_moment_avg}).

\subsection{SGD}
\label{sec:ablation_sgd}

We train several models using SGD with momentum $\gamma = 0.9$ as the optimizer for the embeddings. Since Adam via the inverse square root of its second moment effectively scales the learning rate up by a factor comprising orders of magnitude (see Eq.~(\ref{eq:second_moment_global_factor})), we explicitly multiply the learning rate in SGD by a factor $f$ of comparable size\footnote{Note that the difference between momentum in SGD and the first moment in Adam also plays a role here.}. A hyperparameter search using 
$f \in \{ 100, 200, 300, 400, 500, 600 \}$
is performed to search for the optimum with respect to upstream performance (loss), see App.~\ref{app:additional_results_sgd} for details. It is found at $f = 300$ for $D \in \{ 5\B, 10\B \}$ and $f = 400$ for $D = 20\B$.
The respective optimal model is compared to its counterpart trained with Coupled Adam in Tab.~\ref{tab:results_ablations_sgd_all}.
\begin{table*}[htb]
\centering
\scriptsize
\begin{tabular}{ccc|rrrrrrrr}
\toprule
$D$ & $N$ & Optimizer & $\Loss$ ($\downarrow$) & $\Acc$ ($\uparrow$) & $\ISO$ ($\uparrow$) & $\munorm$ ($\downarrow$) & $\munormrel$ ($\downarrow$) & $\rcos$ ($\uparrow$) & $\rho$ ($\uparrow$) & $\kappa$ ($\uparrow$) \\ 
\midrule
 \multirow{2}{*}{5B   } & \multirow{2}{*}{   125M   } &   SGD (300)   &  3.17 {\tiny (0)}  &  0.333 {\tiny (3)}  & {\bf 0.99 {\tiny (0)}} & {\bf 0.00 {\tiny (0)}} & {\bf 0.01 {\tiny (0)}} &  45 {\tiny (1)}  &  71 {\tiny (1)}  & {\bf 15.5 {\tiny (4)}}\\
                        &                          &   Coupled Adam   & {\bf 3.12 {\tiny (1)}} & {\bf 0.339 {\tiny (2)}} &  0.94 {\tiny (1)}  &  0.02 {\tiny (0)}  &  0.01 {\tiny (0)}  & {\bf 55 {\tiny (0)}} & {\bf 87 {\tiny (1)}} &  4.8 {\tiny (2)}\\ \midrule
 \multirow{2}{*}{10B   } & \multirow{2}{*}{   125M   } &   SGD (300)   &  3.07 {\tiny (1)}  &  0.341 {\tiny (4)}  & {\bf 0.99 {\tiny (0)}} & {\bf 0.01 {\tiny (0)}} & {\bf 0.01 {\tiny (0)}} &  49 {\tiny (1)}  &  71 {\tiny (4)}  & {\bf 12.2 {\tiny (2)}}\\
                        &                          &   Coupled Adam   & {\bf 3.03 {\tiny (0)}} &  0.343 {\tiny (1)}  &  0.91 {\tiny (1)}  &  0.05 {\tiny (0)}  &  0.02 {\tiny (0)}  & {\bf 57 {\tiny (2)}} & {\bf 82 {\tiny (0)}} &  3.6 {\tiny (8)}\\ \midrule
 \multirow{2}{*}{20B   } & \multirow{2}{*}{   125M   } &   SGD (400)   &  3.00 {\tiny (0)}  &  0.346 {\tiny (5)}  & {\bf 0.98 {\tiny (1)}} & {\bf 0.01 {\tiny (0)}} & {\bf 0.02 {\tiny (0)}} &  54 {\tiny (0)}  &  76 {\tiny (5)}  & {\bf 7.4 {\tiny (1.1)}}\\
                        &                          &   Coupled Adam   & {\bf 2.97 {\tiny (0)}} &  0.350 {\tiny (1)}  &  0.83 {\tiny (1)}  &  0.11 {\tiny (0)}  &  0.03 {\tiny (0)}  &  57 {\tiny (2)}  &  77 {\tiny (0)}  &  1.7 {\tiny (5)}\\ 

\bottomrule 
\end{tabular}
\caption{Comparison of models whose embeddings were trained with SGD and Coupled Adam. The SGD models were obtained after hyperparameter search for the learning rate. The associated factor $f$ is specified in parentheses in the Optimizer column. Bold values indicate better results with statistical significance, see the caption of Tab.~\ref{tab:results_S} for more details.}
\label{tab:results_ablations_sgd_all}
\end{table*}
The results show that, although SGD is advantageous with respect to isotropy, the mean embedding shift and the condition number, Coupled Adam consistently achieves better results on all upstream and downstream task metrics, while having one less hyperparameter to fine-tune.

\section{Related Work}
\label{sec:related_work}

\citet{gao2019representationdegenerationproblemtraining}
first described the anisotropy issue, which they referred to as \textit{representation degeneration problem}, and suggested cosine regularization as a mitigation strategy.
Alternative techniques to address the problem have been developed, including adversarial noise \cite{pmlr-v97-wang19f}, spectrum control \cite{Wang2020ImprovingNL} and Laplacian regularization \cite{zhang-etal-2020-revisiting}.  
\citet{bis2021tmic} have shown that the anisotropy of embeddings can for the most part be traced back to a common shift of the embeddings in a dominant direction. They called this phenomenon \textit{common enemy effect}, and provided a semi-quantitative explanation (Eq.~(\ref{eq:chain_rule_e})), which we developed further in the present work by including the optimizer in the analysis.
In \citet{yu-etal-2022-rare}, Adaptive Gradient Gating is proposed, based on the empirical observation that it is the gradients for embeddings of rare tokens that cause anisotropy. Our analysis conforms to this finding and attributes it to a massive up-scaling of the gradients for rare embeddings with Adam, cf.~Fig.~\ref{fig:gradients_example}.
\citet{machina-mercer-2024-anisotropy} have demonstrated that large Pythia models \cite{pmlr-v202-biderman23a} show improved isotropy compared to similar models, and attribute this to the absence of weight tying. This is in accordance with our analysis of the unembedding gradients in conjunction with Adam, Sec.~\ref{sec:theory}.
While all the previously mentioned papers use average cosine similarity \cite{ethayarajh-2019-contextual} or $\ISO$ from Eq.~(\ref{eq:isotropy}) to quantify the geometry of embedding vectors, \citet{rudman-etal-2022-isoscore} deviate from this. 
Their notion of isotropy is based solely on the embeddings' covariance matrix and embodied by the metric IsoScore. 
In particular, IsoScore is mean-agnostic, while $\ISO$ strongly correlates with the mean embedding (see e.g. Tab.~\ref{tab:results_S}).
In a follow-up
paper \citep{rudman2024stableanisotropicregularization}, IsoScore is
used to regularize isotropy, which appears to benefit performance for fine-tuning tasks.
Finally, concurrent to our work, \citet{zhao2024deconstructingmakesgoodoptimizer} have investigated the importance of using the second moment in Adam with regard to performance and stability. They found that simplified variants of Adam that use the same effective learning rate either for the whole embedding matrix (Adalayer) or each embedding vector (Adalayer*) are slightly worse than Adam but better than SGD. Adalayer* is similar to Coupled Adam, but corresponds to the second moment averaged over hidden space instead of vocabulary space.

\section{Conclusions}

Our work addresses the well-known anisotropy problem for LLM embeddings. 
We have advanced the theoretical understanding of the phenomenon by showing that it is a combination of the common enemy effect and the individual second moments in Adam that causes a collective shift of the embedding vectors away from the origin.
To mitigate the problem, we have introduced Coupled Adam, which enforces the same effective learning rate for every embedding vector, and thus suppresses the collective shift of the embeddings.
We have found that Coupled Adam consistently improves embedding-specific metrics across all experiments, while also achieving better downstream and upstream performance for large datasets, as they are typically used in LLM training.
The code to reproduce our results is available at \href{https://github.com/flxst/coupled-adam}{\nolinkurl{github.com/flxst/coupled-adam}}~.

\section{Limitations}

Although our method is generally applicable to all common LLM architectures, as they share the same language modeling head and embeddings, only dense decoders were used in our experiments. 
In addition, only models with up to $N=2.6\B$ parameters have been tested. Our experiments involve pre-training and few-shot downstream evaluation, yet fine-tuning tasks have not been included.
The cosine decay learning rate schedule was applied throughout all experiments (App.~\ref{app:hyperparameters}). Alternatives such as an infinite learning rate schedule are not incorporated in our study.
It would also be interesting to extend our work to optimizers other than SGD and Adam.
Furthermore, as mentioned at the end of Sec.~\ref{sec:results}, we have not explicitly verified that the slight residual shift of the mean embedding, which is observed even for Coupled Adam, is caused by weight tying.
Finally, we have used a straightforward implementation of Coupled Adam, closely following Algorithm~\ref{alg:algorithm_adam}. More sophisticated implementations might lead to increased efficiency and further improvements; we leave it for future work to investigate this.

\section*{Acknowledgements}
Our computational experiments used around 20000 GPU hours. They were partly run on the EuroHPC supercomputers MeluXina and MareNostrum5 in conjunction with the grants EHPC-DEV-2023D10-032 and EHPC-EXT-2023E02-038.

\bibliography{custom}

\appendix
\section{Unigram Probability Distribution}
\label{app:unigram_probability_example}

Fig.~\ref{fig:prior_example_gpt2_openwebtext} shows the unigram probability distribution for the example of the OpenWebText Corpus dataset and the GPT-2 tokenizer.
\begin{figure}[H]
    \centering
    \includegraphics[scale=0.5]{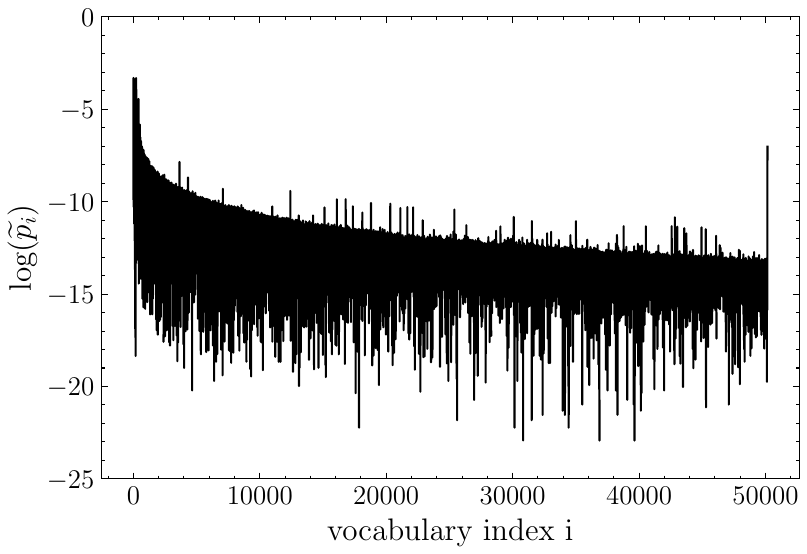}
    \caption{Logarithm $\log(\widetilde{p_i})$ of the unigram probability distribution for the OpenWebText Corpus and the GPT-2 tokenizer. The maximum probability is $\max_{i} \widetilde{p_i} \approx 0.037$ or $\max_{i} \log(\widetilde{p_i}) \approx -3.30$. The minimum probability (not shown) is $\min_{i} \widetilde{p_i} = 0$ or $\min_{i} \log(\widetilde{p_i}) = -\infty$.}
    \label{fig:prior_example_gpt2_openwebtext}
\end{figure}

\section{Embedding Gradients}
\label{app:chain_rule_e}

We explicitly derive Eq.~(\ref{eq:chain_rule_e}), which we recall here for convenience:
\begin{align}
g_i :=~ &\frac{\partial \mathcal{L}}{\partial e_i} 
= - \left( \delta_{it} - p_i \right) \cdot h \tag{\ref{eq:chain_rule_e}}
\end{align}
The chain rule yields
\begin{align}
\frac{\partial \mathcal{L}}{\partial e_i} 
&= \sum_{k=1}^{V} \frac{\partial \mathcal{L}}{\partial p_t} 
\cdot \frac{\partial p_t}{\partial l_k}  
\cdot \frac{\partial l_k}{\partial e_i} \; ,
\label{eq:chain_rule_basis}
\end{align}
where the individual factors can directly be obtained from Eqs.~(\ref{eq:forward_loss})-(\ref{eq:gradient_function_new}):
\begin{align}
\frac{\partial \mathcal{L}}{\partial p_t} &= - \frac{1}{p_t} \label{eq:backward_loss} \\
\frac{\partial p_t}{\partial l_k} 
&= \frac{\delta_{kt} \exp{(l_t)} \cdot \Sigma - \exp{(l_t)} \exp{(l_k)}}{\Sigma^2} \nonumber \\ 
&= \delta_{kt} p_t - p_t p_k \nonumber \\
&= p_t ( \delta_{kt} - p_k ) 
\label{eq:backward_loss_2} \\
\frac{\partial l_k}{\partial e_i} &= \delta_{ki} h \label{eq:backward_e}
\end{align}
Note that in the first line of Eq.~(\ref{eq:backward_loss_2}), we use the abbreviation $\Sigma = \Big( \sum_{j=1}^V \exp{(l_j)} \Big)$.
Inserting Eqs.~(\ref{eq:backward_loss}), (\ref{eq:backward_loss_2}) and (\ref{eq:backward_e}) into Eq.~(\ref{eq:chain_rule_basis}) directly leads to Eq.~(\ref{eq:chain_rule_e}):
\begin{align}
\frac{\partial \mathcal{L}}{\partial e_i} 
&= - \sum_{k=1}^{V} \frac{1}{p_t}
\cdot p_t ( \delta_{kt} - p_k )  
\cdot \delta_{ki} h \nonumber \\
&= - \sum_{k=1}^{V} ( \delta_{kt} - p_k )  
\cdot \delta_{ki} h \nonumber \\
&= - ( \delta_{it} - p_i )  
\cdot h \nonumber
\end{align}

\section{SGD Algorithm}
\label{app:sgd_algorithm}

For completeness and comparison to (Coupled) Adam as displayed in Algorithm~\ref{alg:algorithm_adam}, we summarize the SGD algorithm in Algorithm~\ref{alg:algorithm_sgd}.

\begin{algorithm}[H]
    \small
    \textbf{Input:}
    $\eta$ (lr), $e_i^{(0)}$ (initial embeddings),
    $\mathcal{L}(e_i)$ (objective), $\gamma$ (momentum), $T$ (number of time steps) \\
    \textbf{Output}: $e^{(T)}$ (final embeddings)
    \begin{algorithmic}[1]
        \For{$\tm=1 \dots T$}
            \For{$i=1 \dots V$}
                \State $g_i^{(\tm)}$ $\gets$ $\nabla_{e_i} \mathcal{L}^{(\tm)} (e_i^{(\tm-1)})$
                \If{$t>1$}
                    \State $\mathbf{b}_i^{(\tm)}$ $\gets$ $\gamma \mathbf{b}_i^{(\tm-1)} + g_i^{(\tm)}$
                \Else
                    \State $\mathbf{b}_i^{(\tm)}$ $\gets$ $g_i^{(\tm)}$
                \EndIf
                \State $e_i^{(\tm)}$ $\gets$ $e_i^{(\tm-1)} - \eta \mathbf{b}_i^{(\tm)}$
            \EndFor
        \EndFor
        \vspace*{1.0ex}
        \State \Return $e^{(T)}$
    \end{algorithmic}
    \caption{Pseudocode for the SGD algorithm with optional momentum, applied to the embedding vectors $e_i$.}
    \label{alg:algorithm_sgd}
\end{algorithm}

\section{Magnitude of the Second Moment in Adam}

In this appendix, the validity of 
\begin{equation}
    \E \left[ \secondmomentshort \right] \propto \widetilde p_i
    \tag{\ref{eq:second_moment_linear_in_unigram_prob}}
\end{equation}
is verified.
Due to the linearity of lines 5 and 7 in Algorithm 2, it suffices to show that the squared gradient has the property in question:
\begin{align}
\E \left[ g_i^2 \right] &\propto \widetilde p_i
\label{eq:squared_gradient_linear_in_unigram_prob}
\end{align}
We do this in two different ways.
First, we derive Eq.~(\ref{eq:squared_gradient_linear_in_unigram_prob}) using a semi-theoretical approach with minimal experimental input.
Afterwards, we confirm the relationship in a purely experimental manner. %

\subsection{Semi-theoretical Derivation}
\label{app:second_moment_theory}

Here, we derive an expression for the expectation value of the squared gradient in terms of simple observables (Theorem~\ref{theorem}). Subsequently, the dependency of those observables on $\widetilde p_i$ is determined experimentally. Together, this will yield the proportionality expressed by Eq.~(\ref{eq:squared_gradient_linear_in_unigram_prob}).
We begin our reasoning with a lemma.

\begin{lemma}[Expectation Value Decomposition]\label{lemma}
The expectation value of the squared gradient can be decomposed into conditional expectation values as follows:
\begin{align}
\E \left[ g_i^2 \right] =~ 
&\widetilde p_i \cdot \E \left[ g_i^2 ~\big|~ i=t \right] \nonumber \\
&+ (1 - \widetilde p_i) \cdot \E \left[ g_i^2 ~\big|~ i \neq t \right]
\label{eq:lemma1}
\end{align}
\end{lemma}

\begin{proof}
Our starting point is the definition of the expectation value for the continuous random variable $g_i^2$:
\begin{align}
\E \left[ g_i^2 \right] = \int g_i^2 ~p(g_i) ~dg_i \; ,
\label{eq:aux_E_definition}
\end{align}
where $p$ denotes the probability distribution of $g_i$. Since the vocabulary item $i$ can only be either the true token $t$ or not, we can decompose $p$ into a sum of joint probability distributions (using the {\em law of total probabilities}), each of which can be expressed in terms of conditional probabilities like so:
\begin{align}
p(g_i) 
&= p(g_i, i=t) + p(g_i, i\neq t) \nonumber \\
&= p(g_i ~|~ i=t) \cdot p(i=t) \nonumber \\
&\quad + p(g_i ~|~ i\neq t) \cdot p(i\neq t)
\end{align}
Using the unigram probability $\widetilde p_i = p(i = t)$, this can also be written as
\begin{align}
p(g_i) 
&= \widetilde p_i \cdot p(g_i ~|~ i=t) \nonumber \\
&\quad + (1 - \widetilde p_i) \cdot p(g_i ~|~ i\neq t)
\label{eq:aux_p_decomposition}
\end{align}
If we insert Eq.~(\ref{eq:aux_p_decomposition}) back into Eq.~(\ref{eq:aux_E_definition}), the expectation value becomes
\begin{align}
\E \left[ g_i^2 \right] 
&= \widetilde p_i \cdot \int g_i^2 ~p(g_i ~|~ i=t) ~dg_i \nonumber \\
&\quad + (1 - \widetilde p_i) \cdot \int g_i^2 ~p(g_i ~|~ i\neq t) ~dg_i \;, 
\label{eq:aux_E_decomposition} %
\end{align}
which by definition of the (conditional) expectation value, Eq.~(\ref{eq:aux_E_definition}), is equivalent to Eq.~(\ref{eq:lemma1}).
\end{proof}
\begin{theorem}[Expectation Value Squared Gradient] \label{theorem}
Given that the squared hidden state vector $h^2$ is independent of $p_i$ and whether $i$ is the true token or not, the expectation value of the squared gradient $g_i^2$ is given by
\begin{align}
\E \left[ g_i^2 \right] 
&= S \cdot \left[ \widetilde p_i \cdot X_i^{(i=t)} + (1 - \widetilde p_i) \cdot X_i^{(i \neq t)} \right] \; ,
\label{eq:theorem}
\end{align}
with
\begin{align}
S &:= \E \left[ h^2 \right] \label{eq:optimizer_second_moment_expansion_S} \\
X_i^{(i=t)} &:= \E \left[ (1 - p_i)^2 ~\big|~ i=t \right]
\label{eq:optimizer_second_moment_expansion_true_2} \\
X_i^{(i \neq t)} &:= \E \left[ p_i^2 ~\big|~ i \neq t \right]
\label{eq:optimizer_second_moment_expansion_false_2}
\end{align}
\end{theorem}

\begin{proof}
We start from Lemma~\ref{lemma} and the square of the gradient,
\begin{align}
g_i^2 
&\stackrel{(\ref{eq:chain_rule_e})}{=} \left( \delta_{it} - p_i \right)^2 h^2 
\label{eq:optimizer_second_moment}
\end{align}
Note that squared variables of vectors in $\mathbb{R}^H$ always denote the elementwise (Hadamard) product, e.g.
\begin{align}
g_i^2 &\equiv g_i \odot g_i \in \mathbb{R}_{\geq 0}^H \; ,
\label{eq:hadamard_product}
\end{align}
with strictly non-negative elements.
Using Eq.~(\ref{eq:optimizer_second_moment}), the expectation values on the right side of Eq.~(\ref{eq:lemma1}) can be expressed as
\begin{align}
\E \left[ g_i^2 ~\big|~ i=t \right]
&= \E \left[ \left( 1 - p_i \right)^2 \cdot h^2 ~\big|~ i=t \right] \\
\E \left[ g_i^2 ~\big|~ i\neq t \right]
&= \E \left[ p_i^2 \cdot h^2 ~\big|~ i\neq t \right]
\end{align}
Given our assumptions regarding $h^2$, its expectation value can be factored out:
\begin{align}
\E \left[ g_i^2 ~\big|~ i=t \right]
&= S \cdot X_i^{(i=t)} \label{eq:Etrue} \\
\E \left[ g_i^2 ~\big|~ i\neq t \right]
&= S \cdot X_i^{(i \neq t)} \label{eq:Efalse} 
\end{align}
Inserting Eqs.~(\ref{eq:Etrue}) and (\ref{eq:Efalse}) into Eq.~(\ref{eq:lemma1}) yields Eq.~(\ref{eq:theorem}).
\end{proof}

Note that Eq.~(\ref{eq:theorem}) is a vector equation, with $\E \left[ g_i^2 \right], S \in \mathbb{R}_{\geq 0}^H$ and $\widetilde p_i, X_i^{(i=t)}, X_i^{(i \neq t)} \in \mathbb{R}_{\geq 0}$.
It states that the expectation value of $g_i^2$ factorizes into a global constant $S$ that is $i$-independent, and a factor that is $i$-dependent. The latter is a specific combination of the unigram probability $\widetilde p_i$, determined by the data, and the conditional expectation values $X_i^{(i=t)}$ and $X_i^{(i\neq t)}$, determined by the model.

\paragraph{Experimental Input}

Regarding the unigram probability, we know that
\begin{enumerate}
\item $\widetilde p_i \ll 1$. \\
This is the case for virtually all natural language datasets with a common vocabulary size of $V > 10000$, according to Zipf's law.
\end{enumerate}
The conditional expectation values $X_i^{(i=t)}$ and $X_i^{(i\neq t)}$ can be empirically estimated by applying training data to different checkpoints. We consider the three small-scale experiments of Sec.~\ref{sec:experiments_S} with $N \in \{ 125\M, 355\M, 760\M \}$ and $D=20\B$, and take ten equidistant checkpoints after $D^\prime \in \{ 2\B, 4\B, \ldots, 20\B\}$ seen tokens for each of them. We then continue pseudo-training on 20 batches ($\approx$ 2k samples or 2M tokens, see Tab.~\ref{tab:model_architecture}) of data using a zero learning rate, and measure the conditional probabilities in Eqs.~(\ref{eq:optimizer_second_moment_expansion_true_2}, \ref{eq:optimizer_second_moment_expansion_false_2}) from which our target quantities can be estimated.  
Subsequently, linear fits of the form 
\begin{align}
    X_i^{(i = t)} &= A^{(i = t)} \cdot \widetilde p_i \\
    X_i^{(i \neq t)} &= A^{(i \neq t)} \cdot \widetilde p_i \; ,
\end{align}
with fit parameters $A^{(i = t)}$ and $A^{(i \neq t)}$ are performed. $R^2$ is used to assess the quality of the fits. In addition, the mutual information $\I$ between the response and the explanatory variable is computed. 
Since we observe only a very weak dependence of the results for $R^2$ and $\I$ on $N$ and $D^\prime$, we specify the mean and standard deviation over all experiments for them.
Our findings are:
\begin{enumerate}
\item[2.] $X_i^{(i = t)}$ is independent of $\widetilde p_i$. \\
The linear fits yield $R^2 = 0.003(1)$, and the mutual information is $\I \left(X_i^{(i = t)}; \widetilde p_i \right) = 0.14(2)$.
\item[3.] $X_i^{(i \neq t)}$ is proportional to $\widetilde p_i$. \\
The linear fits yield $R^2 = 0.92(1)$, and the mutual information is $\I \left( X_i^{(i \neq t)}; \widetilde p_i \right) = 0.50(2)$.
\end{enumerate}

The three empirical results above, together with Theorem~\ref{theorem}, immediately lead to Eq.~(\ref{eq:squared_gradient_linear_in_unigram_prob}).

\subsection{Experimental Confirmation}
\label{app:second_moment_empirical}

We reuse the experiments from the previous section to measure the second moment $\secondmomentshort$ directly, in order to estimate $\E \left[ \secondmomentshort \right]$. Again, linear fits of the form
\begin{align}
    \E \left[ \secondmomentshort \right] = A \cdot \widetilde p_i
\end{align}
are performed and the mutual information is computed. 
We find
\begin{enumerate}
\item[4.] $\E \left[ \secondmomentshort \right]$ is proportional to $\widetilde p_i$. \\
The linear fits yield $R^2 = 0.85(7)$, and the mutual information is $\I \left( \E \left[ \secondmomentshort \right]; \widetilde p_i \right) = 1.18(9)$.
\end{enumerate}
The results for $N=125\M$ and $D=D^\prime=20\B$ are depicted in Fig.~\ref{fig:experimental_results_E_p}, as an example.
\begin{figure}[h!]
    \centering
    \includegraphics[scale=0.5]{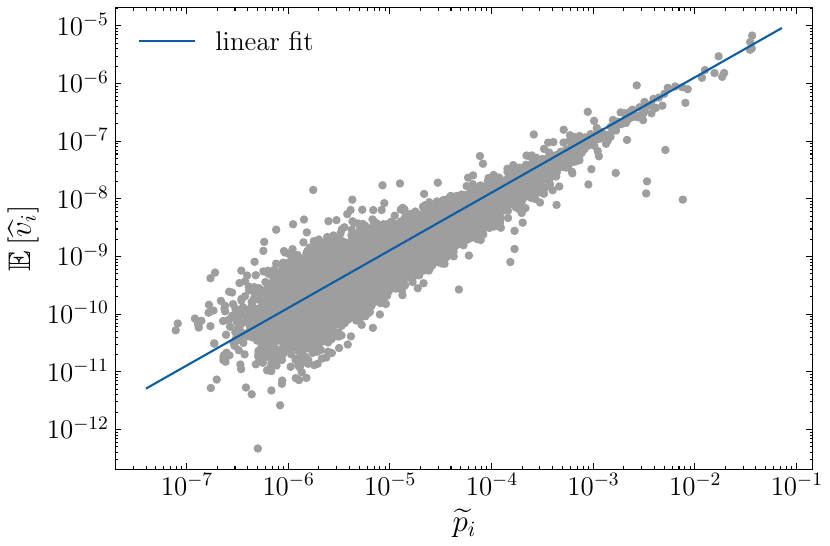}
    \caption{Experimental results for $\E \left[ \secondmomentshort \right]$ (vertical axis) vs. $\widetilde p_i$ (horizontal axis) for $N=125\M$ and $D=D^\prime=20\B$. The blue line shows the linear fit with $R^2 = 0.91$.}
    \label{fig:experimental_results_E_p}
\end{figure}

Note that while $R^2$ and $\I$ are again virtually independent of $N$ and $D^\prime$, the fit parameter $A$ is not. Instead, it seems to increase with $D^\prime$, as shown in Fig.~\ref{fig:experimental_results_A}.
\begin{figure}[h!]
    \centering
    \includegraphics[scale=0.5]{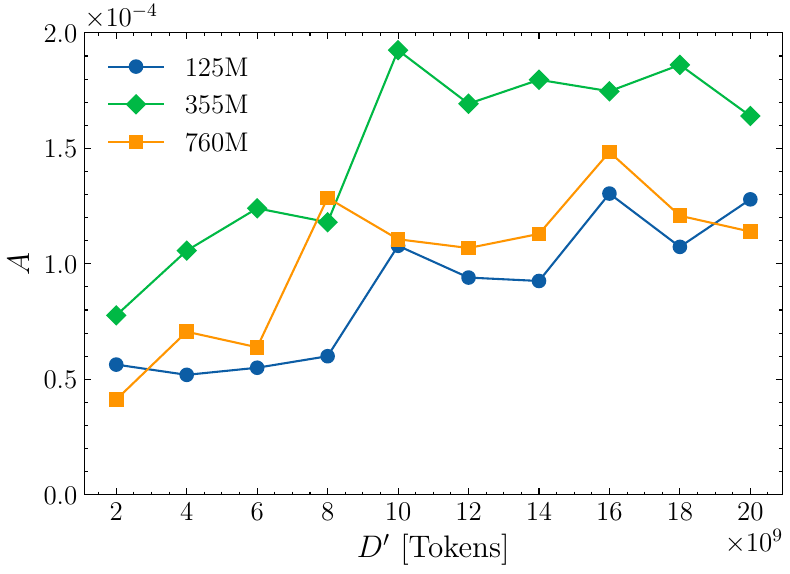}
    \caption{Experimental results for the linear fit parameter $A$ as a function of $N$ and $D^\prime$.}
    \label{fig:experimental_results_A}
\end{figure}
However, as stated in Eq.~(\ref{eq:second_moment_proportionality_constant}), the order of magnitude is $A \approx 10^{-4}$ throughout our experiments.

\section{Experimental Details}
\subsection{Model and Dataset Sizes}
\label{app:experiments_overview}

The model sizes $N$ and dataset sizes $D$ employed in our experiments are depicted in Fig.~\ref{fig:experiments}.
\begin{figure}[ht]
    \centering
    \includegraphics[scale=0.5]{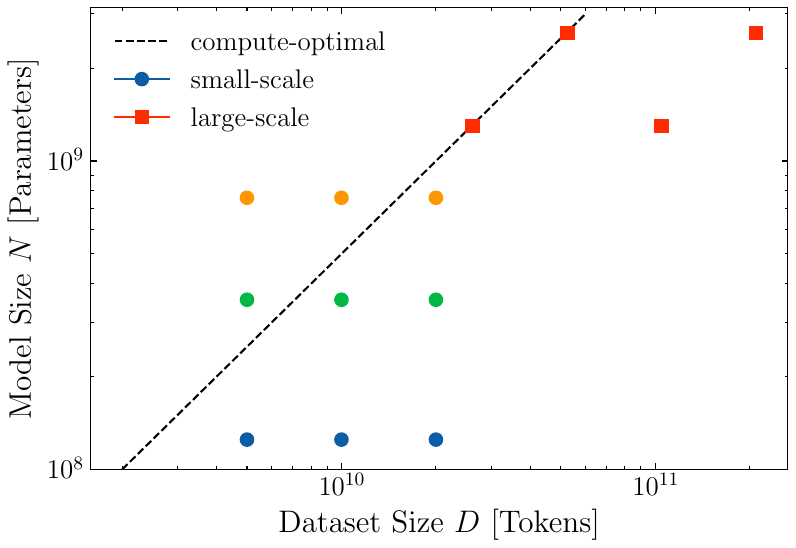}
    \caption{Overview of the dataset (horizontal axis) and model sizes (vertical axis) involved in our small-scale (blue, green and orange circles) and large-scale (red squares) experiments. The dashed, black line shows $N = D / 20$, which is approximately the compute-optimal trajectory according to \citet{hoffmann2022trainingcomputeoptimallargelanguage}.}
    \label{fig:experiments}
\end{figure}

\subsection{Training Hyperparameters}
\label{app:hyperparameters}

In Tab.~\ref{tab:model_architecture}, we list the general hyperparameters used in our small-scale (Sec.~\ref{sec:experiments_S}) and large-scale (Sec.~\ref{sec:experiments_L}) experiments. 
\begin{table}[ht]
\centering
\scriptsize
\begin{tabular}{l|cc}
\toprule
Description & Small-scale & Large-scale \\ 
\midrule
optimizer & \multicolumn{2}{c}{AdamW} \\ 
$\beta_1$ & \multicolumn{2}{c}{0.9} \\
$\beta_2$ & \multicolumn{2}{c}{0.95} \\
$\epsilon$ & \multicolumn{2}{c}{1e-8} \\
weight decay & \multicolumn{2}{c}{0.1} \\
gradient clipping & \multicolumn{2}{c}{1.0} \\
dropout & \multicolumn{2}{c}{0.0} \\
weight tying & \multicolumn{2}{c}{true} \\
vocab size & \multicolumn{2}{c}{50304} \\
learning rate schedule & \multicolumn{2}{c}{cosine decay} \\
layer normalization & \multicolumn{2}{c}{LayerNorm} \\
precision & \multicolumn{2}{c}{BF16} \\
\midrule
hidden activation & GeLU & SwiGLU \\
positional embedding & absolute (learned) & RoPE \\ 
sequence length & 1024 & 2048 \\
batch size (samples) & 96 & 256 \\
batch size (tokens) & $\sim$100k & $\sim$500k \\
warmup & 100 steps & $1\%$ of steps \\
training framework & nanoGPT & Modalities \\
training parallelism & DDP & FSDP \\
\bottomrule 
\end{tabular}
\caption{General hyperparameters used in our two sets of experiments.}
\label{tab:model_architecture}
\end{table}
During warm-up, the learning rate is increased from zero to the maximum learning rate. This is followed by a cosine decay which reduces the learning rate to $10\%$ of the maximum at the end of training. Note that weight decay is applied only to linear layers, not layer norms or embeddings.
Tab.~\ref{tab:model_sizes} shows the hyperparameters related to model size, following GPT-3 \cite{brown2020languagemodelsfewshotlearners}.
\begin{table}[ht]
\centering
\scriptsize
\begin{tabular}{c|cccc}
\toprule
$N$ & lr & heads & layers & emb. dim. \\ 
\midrule
125M & 6.0e-4 & 12 & 12 & 768 \\
355M & 3.0e-4 & 16 & 24 & 1024 \\
760M & 2.5e-4 & 16 & 24 & 1536 \\
1.3B & 2.0e-4 & 32 & 24 & 2048 \\
2.6B & 1.6e-4 & 32 & 32 & 2560 \\
\bottomrule 
\end{tabular}
\caption{Model-size dependent hyperparameters used in our experiments. $N$ denotes the model size in terms of parameters, while lr corresponds to the maximum learning rate.}
\label{tab:model_sizes}
\end{table}

\section{Error Analysis and Statistical Significance}
\label{app:error}

For the error analysis, two separate random variables, $X_0$ and $X_1$, are considered. The symbol $X$ represents one of the metrics discussed in Sec.~\ref{sec:experiments_evaluation}, while $0$ and $1$ stand for two approaches that are to be compared, like standard Adam and Coupled Adam, for instance.
For each of the two random variables $i = \{ 0, 1 \}$, we conduct and evaluate $S$ training runs with different seeds, yielding results 
\begin{align}
\{ X_i^{(1)}, \ldots, X_i^{(S)} \}
\end{align}
While it is desirable to have a large sample size $S$, it is  prohibitively expensive for large model and dataset sizes to repeat training runs. We use
\begin{align}
S &= 3
\label{eq:error_S3}
\end{align}
except for the large-scale experiments (Sec.~\ref{sec:experiments_L}), where we restrict ourselves to
\begin{align}
S &= 1
\label{eq:error_S1}
\end{align}
We are interested in the difference 
\begin{align}
d = X_1 - X_0
\label{eq:error_d}
\end{align}
For $S=1$, it can be computed straight forwardly. However, no statement about the statistical uncertainty or significance of $d$ can be made.
In the case of $S = 3$, we apply a one-sided Student's t-test with a confidence level of 
\begin{align}
\alpha = 95\%
\label{eq:error_confidence_level_alpha}
\end{align}
First, the sample means
\begin{align}
\bar X_i &= \frac{1}{S} \sum_{s=1}^S X_i^{(s)}
\label{eq:student_mean_i}
\end{align}
and the corrected sample standard deviations
\begin{align}
\hat \sigma_i^2 &= \frac{1}{S-1} \sum_{s=1}^S \left( X_i^{(s)} - \bar X_i \right)^2
\label{eq:student_std_i}
\end{align}
for the two samples $i \in \{0, 1\}$ are estimated.
The sample means from Eq.~(\ref{eq:student_mean_i}) are combined to an estimate for their difference,
\begin{align}
\bar d &= \bar X_1 - \bar X_0
\label{eq:student_mean_d}
\end{align}
and the sample standard deviations from Eq.~(\ref{eq:student_std_i}) are propagated to the sample standard deviation of $d$ via Gaussian error propagation:
\begin{align}
\hat \sigma_d &= 
\sqrt{\left( \frac{\partial d}{\partial X_0} \cdot \hat \sigma_0 \right)^2 + \left( \frac{\partial d}{\partial X_1} \cdot \hat \sigma_1 \right)^2} \nonumber \\
&\stackrel{(\ref{eq:error_d})}{=} \sqrt{\hat \sigma_0^2 + \hat \sigma_1^2}
\label{eq:sigmad}
\end{align}

Student's t-distribution for the chosen confidence level $\alpha$ (see Eq.~(\ref{eq:error_confidence_level_alpha})) and the
\begin{align}
\nu &= S - 1 \stackrel{(\ref{eq:error_S3})}{=} 2 
\end{align}
degrees of freedom yields
\begin{align}
t_{\alpha, \nu} = 2.92
\label{eq:talphanu}
\end{align}
With $S$, $\sigma_d$ and $t_{\alpha, \nu}$ from Eqs.~(\ref{eq:error_S3}), (\ref{eq:sigmad}) and (\ref{eq:talphanu}) as ingredients, the one-sided confidence threshold for the difference can be computed as
\begin{align}
d_{\rm significance}
&= t_{\alpha, \nu} \cdot \frac{\hat \sigma_d}{\sqrt{S}}
\label{eq:studentonesidedconfidencethreshold}
\end{align}
Hence, the estimate $\bar d$ from Eq.~(\ref{eq:student_mean_d}) is considered a statistically significant improvement of approach $i=1$ over approach $i=0$ if
\begin{align}
\bar d < - d_{\rm significance}
\label{eq:student_signficance_loss}
\end{align}
for metrics where smaller values are desirable (e.g. $\Loss$), and 
\begin{align}
\bar d > d_{\rm significance}
\label{eq:student_signficance_other}
\end{align}
for metrics where larger values are better (e.g. $\Acc$).

\newcommand{\resultsAblationsSGDExpFive}{ \multirow{8}{*}{5B   } & \multirow{8}{*}{   125M   } &   SGD (100)   &  3.23 {\tiny (0)}  &  0.325 {\tiny (1)}  &  1.00 {\tiny (0)}  & {\bf 0.00 {\tiny (0)}} & {\bf 0.01 {\tiny (0)}} &  33 {\tiny (1)}  &  59 {\tiny (1)}  & {\bf 25.0 {\tiny (1.1)}}\\
                        &                        &   SGD (200)   &  3.19 {\tiny (1)}  &  0.327 {\tiny (2)}  &  1.00 {\tiny (0)}  &  0.00 {\tiny (0)}  &  0.01 {\tiny (0)}  &  41 {\tiny (1)}  &  66 {\tiny (1)}  &  19.2 {\tiny (9)}\\
                        &                        &   SGD (300)   &  3.17 {\tiny (0)}  &  0.333 {\tiny (3)}  &  0.99 {\tiny (0)}  &  0.00 {\tiny (0)}  &  0.01 {\tiny (0)}  &  45 {\tiny (1)}  &  71 {\tiny (1)}  &  15.5 {\tiny (4)}\\
                        &                        &   SGD (400)   &  3.18 {\tiny (2)}  &  0.332 {\tiny (3)}  &  0.99 {\tiny (0)}  &  0.01 {\tiny (0)}  &  0.01 {\tiny (0)}  &  48 {\tiny (3)}  &  75 {\tiny (1)}  &  14.2 {\tiny (5)}\\
                        &                        &   SGD (500)   &  3.19 {\tiny (2)}  &  0.330 {\tiny (6)}  &  0.99 {\tiny (0)}  &  0.01 {\tiny (0)}  &  0.01 {\tiny (0)}  &  49 {\tiny (1)}  &  77 {\tiny (1)}  &  13.2 {\tiny (2)}\\
                        &                        &   SGD (600)   &  3.24 {\tiny (2)}  &  0.326 {\tiny (3)}  &  0.99 {\tiny (0)}  &  0.01 {\tiny (0)}  &  0.01 {\tiny (0)}  &  48 {\tiny (2)}  &  79 {\tiny (1)}  &  11.6 {\tiny (2)}\\ \cmidrule{3-11}
                        &                        &   Standard Adam   &  3.14 {\tiny (0)}  &  0.340 {\tiny (2)}  &  0.31 {\tiny (2)}  &  1.10 {\tiny (6)}  &  0.67 {\tiny (3)}  &  15 {\tiny (3)}  &  -54 {\tiny (3)}  &  0.6 {\tiny (1)}\\
                        &                        &   Coupled Adam   & {\bf 3.12 {\tiny (1)}} &  0.339 {\tiny (2)}  &  0.94 {\tiny (1)}  &  0.02 {\tiny (0)}  &  0.01 {\tiny (0)}  & {\bf 55 {\tiny (0)}} & {\bf 87 {\tiny (1)}} &  4.8 {\tiny (2)}\\
}
\newcommand{\resultsAblationsSGDExpTen}{ \multirow{8}{*}{10B   } & \multirow{8}{*}{   125M   } &   SGD (100)   &  3.13 {\tiny (1)}  &  0.337 {\tiny (4)}  &  1.00 {\tiny (0)}  & {\bf 0.00 {\tiny (0)}} & {\bf 0.01 {\tiny (0)}} &  41 {\tiny (1)}  &  58 {\tiny (4)}  & {\bf 22.2 {\tiny (7)}}\\
                        &                        &   SGD (200)   &  3.09 {\tiny (0)}  &  0.339 {\tiny (4)}  &  0.99 {\tiny (0)}  &  0.01 {\tiny (0)}  &  0.01 {\tiny (0)}  &  48 {\tiny (1)}  &  65 {\tiny (4)}  &  16.1 {\tiny (3)}\\
                        &                        &   SGD (300)   &  3.07 {\tiny (1)}  &  0.341 {\tiny (4)}  &  0.99 {\tiny (0)}  &  0.01 {\tiny (0)}  &  0.01 {\tiny (0)}  &  49 {\tiny (1)}  &  71 {\tiny (4)}  &  12.2 {\tiny (2)}\\
                        &                        &   SGD (400)   &  3.07 {\tiny (0)}  &  0.340 {\tiny (3)}  &  0.99 {\tiny (1)}  &  0.01 {\tiny (0)}  &  0.01 {\tiny (0)}  &  51 {\tiny (2)}  &  76 {\tiny (3)}  &  10.7 {\tiny (8)}\\
                        &                        &   SGD (500)   &  3.09 {\tiny (0)}  &  0.338 {\tiny (3)}  &  0.99 {\tiny (0)}  &  0.01 {\tiny (0)}  &  0.01 {\tiny (0)}  &  52 {\tiny (1)}  &  79 {\tiny (3)}  &  10.3 {\tiny (2)}\\
                        &                        &   SGD (600)   &  3.11 {\tiny (1)}  &  0.341 {\tiny (4)}  &  0.98 {\tiny (1)}  &  0.01 {\tiny (0)}  &  0.01 {\tiny (0)}  &  53 {\tiny (0)}  &  81 {\tiny (1)}  &  9.4 {\tiny (7)}\\ \cmidrule{3-11}
                        &                        &   Standard Adam   &  3.07 {\tiny (0)}  &  0.343 {\tiny (3)}  &  0.21 {\tiny (3)}  &  1.58 {\tiny (5)}  &  0.75 {\tiny (0)}  &  9 {\tiny (2)}  &  -64 {\tiny (5)}  &  0.4 {\tiny (0)}\\
                        &                        &   Coupled Adam   & {\bf 3.03 {\tiny (0)}} &  0.343 {\tiny (1)}  &  0.91 {\tiny (1)}  &  0.05 {\tiny (0)}  &  0.02 {\tiny (0)}  & {\bf 57 {\tiny (2)}} &  82 {\tiny (0)}  &  3.6 {\tiny (8)}\\
}
\newcommand{\resultsAblationsSGDExpTwenty}{ \multirow{8}{*}{20B   } & \multirow{8}{*}{   125M   } &   SGD (100)   &  3.05 {\tiny (1)}  &  0.343 {\tiny (2)}  &  0.99 {\tiny (0)}  & {\bf 0.01 {\tiny (0)}} & {\bf 0.01 {\tiny (0)}} &  50 {\tiny (1)}  &  57 {\tiny (7)}  & {\bf 18.4 {\tiny (4)}}\\
                        &                        &   SGD (200)   &  3.02 {\tiny (0)}  &  0.345 {\tiny (1)}  &  0.99 {\tiny (0)}  &  0.01 {\tiny (0)}  &  0.01 {\tiny (0)}  &  52 {\tiny (1)}  &  64 {\tiny (7)}  &  12.7 {\tiny (7)}\\
                        &                        &   SGD (300)   &  3.01 {\tiny (1)}  &  0.350 {\tiny (1)}  &  0.98 {\tiny (1)}  &  0.01 {\tiny (0)}  &  0.02 {\tiny (0)}  &  53 {\tiny (2)}  &  70 {\tiny (6)}  &  9.0 {\tiny (2)}\\
                        &                        &   SGD (400)   &  3.00 {\tiny (0)}  &  0.346 {\tiny (5)}  &  0.98 {\tiny (1)}  &  0.01 {\tiny (0)}  &  0.02 {\tiny (0)}  &  54 {\tiny (0)}  &  76 {\tiny (5)}  &  7.4 {\tiny (1.1)}\\
                        &                        &   SGD (500)   &  3.01 {\tiny (1)}  &  0.348 {\tiny (4)}  &  0.98 {\tiny (1)}  &  0.02 {\tiny (0)}  &  0.02 {\tiny (0)}  &  55 {\tiny (1)}  &  78 {\tiny (5)}  &  7.5 {\tiny (1.6)}\\
                        &                        &   SGD (600)   &  3.04 {\tiny (3)}  &  0.348 {\tiny (5)}  &  0.95 {\tiny (3)}  &  0.02 {\tiny (0)}  &  0.02 {\tiny (0)}  &  52 {\tiny (5)}  &  81 {\tiny (5)}  &  6.8 {\tiny (2.8)}\\ \cmidrule{3-11}
                        &                        &   Standard Adam   &  3.03 {\tiny (0)}  &  0.346 {\tiny (1)}  &  0.10 {\tiny (3)}  &  2.14 {\tiny (7)}  &  0.82 {\tiny (2)}  &  5 {\tiny (1)}  &  -66 {\tiny (2)}  &  0.3 {\tiny (0)}\\
                        &                        &   Coupled Adam   & {\bf 2.97 {\tiny (0)}} &  0.350 {\tiny (1)}  &  0.83 {\tiny (1)}  &  0.11 {\tiny (0)}  &  0.03 {\tiny (0)}  &  57 {\tiny (2)}  &  77 {\tiny (0)}  &  1.7 {\tiny (5)}\\
}
\newcommand{\resultsSdownstream}{ \multirow{6}{*}{5B   } & \multirow{2}{*}{   125M   } &   Standard   &  0.41 {\tiny (0)}  &  0.19 {\tiny (1)}  &  0.27 {\tiny (0)}  &  0.26 {\tiny (0)}  &  0.28 {\tiny (1)}  &  0.45 {\tiny (0)}  &  0.52 {\tiny (1)}  &  0.340 {\tiny (2)}\\
                        &                          &   Coupled   &  0.40 {\tiny (0)}  &  0.18 {\tiny (1)}  &  0.28 {\tiny (0)}  & {\bf 0.27 {\tiny (0)}} &  0.28 {\tiny (0)}  &  0.45 {\tiny (1)}  &  0.52 {\tiny (1)}  &  0.339 {\tiny (2)}\\ \cmidrule{2-11}
                        & \multirow{2}{*}{   355M   } &   Standard   &  0.43 {\tiny (1)}  &  0.19 {\tiny (1)}  &  0.29 {\tiny (0)}  &  0.31 {\tiny (0)}  &  0.29 {\tiny (0)}  &  0.43 {\tiny (0)}  &  0.52 {\tiny (1)}  &  0.352 {\tiny (3)}\\
                        &                          &   Coupled   &  0.43 {\tiny (2)}  &  0.19 {\tiny (1)}  &  0.29 {\tiny (0)}  &  0.31 {\tiny (1)}  &  0.29 {\tiny (0)}  &  0.43 {\tiny (1)}  &  0.51 {\tiny (0)}  &  0.350 {\tiny (4)}\\ \cmidrule{2-11}
                        & \multirow{2}{*}{   760M   } &   Standard   &  0.46 {\tiny (1)}  &  0.20 {\tiny (1)}  &  0.30 {\tiny (0)}  &  0.34 {\tiny (1)}  &  0.29 {\tiny (1)}  &  0.43 {\tiny (1)}  &  0.50 {\tiny (1)}  &  0.360 {\tiny (3)}\\
                        &                          &   Coupled   &  0.45 {\tiny (0)}  &  0.20 {\tiny (1)}  &  0.30 {\tiny (0)}  &  0.34 {\tiny (1)}  &  0.29 {\tiny (1)}  &  0.42 {\tiny (1)}  &  0.50 {\tiny (0)}  &  0.357 {\tiny (3)}\\ \midrule
 \multirow{6}{*}{10B   } & \multirow{2}{*}{   125M   } &   Standard   &  0.42 {\tiny (1)}  &  0.19 {\tiny (0)}  &  0.28 {\tiny (0)}  &  0.29 {\tiny (1)}  &  0.28 {\tiny (1)}  &  0.44 {\tiny (1)}  &  0.51 {\tiny (1)}  &  0.343 {\tiny (3)}\\
                        &                          &   Coupled   &  0.42 {\tiny (1)}  &  0.19 {\tiny (1)}  & {\bf 0.28 {\tiny (0)}} &  0.29 {\tiny (0)}  &  0.28 {\tiny (1)}  &  0.44 {\tiny (0)}  &  0.50 {\tiny (1)}  &  0.343 {\tiny (1)}\\ \cmidrule{2-11}
                        & \multirow{2}{*}{   355M   } &   Standard   &  0.45 {\tiny (1)}  &  0.19 {\tiny (1)}  &  0.30 {\tiny (0)}  &  0.35 {\tiny (0)}  &  0.29 {\tiny (1)}  &  0.41 {\tiny (0)}  &  0.51 {\tiny (1)}  &  0.359 {\tiny (2)}\\
                        &                          &   Coupled   &  0.46 {\tiny (1)}  &  0.19 {\tiny (1)}  &  0.30 {\tiny (0)}  &  0.35 {\tiny (1)}  &  0.30 {\tiny (1)}  & {\bf 0.42 {\tiny (0)}} &  0.52 {\tiny (1)}  & {\bf 0.365 {\tiny (2)}}\\ \cmidrule{2-11}
                        & \multirow{2}{*}{   760M   } &   Standard   &  0.47 {\tiny (0)}  &  0.21 {\tiny (0)}  &  0.32 {\tiny (0)}  & {\bf 0.39 {\tiny (1)}} &  0.30 {\tiny (1)}  &  0.41 {\tiny (0)}  &  0.53 {\tiny (1)}  &  0.375 {\tiny (2)}\\
                        &                          &   Coupled   &  0.48 {\tiny (1)}  &  0.21 {\tiny (1)}  &  0.32 {\tiny (0)}  &  0.38 {\tiny (0)}  &  0.30 {\tiny (1)}  &  0.41 {\tiny (1)}  &  0.51 {\tiny (1)}  &  0.372 {\tiny (3)}\\ \midrule
 \multirow{6}{*}{20B   } & \multirow{2}{*}{   125M   } &   Standard   &  0.43 {\tiny (0)}  &  0.18 {\tiny (1)}  &  0.28 {\tiny (0)}  &  0.29 {\tiny (1)}  &  0.28 {\tiny (1)}  &  0.44 {\tiny (1)}  &  0.51 {\tiny (1)}  &  0.346 {\tiny (1)}\\
                        &                          &   Coupled   &  0.44 {\tiny (1)}  &  0.19 {\tiny (1)}  &  0.29 {\tiny (0)}  & {\bf 0.31 {\tiny (0)}} &  0.29 {\tiny (1)}  &  0.43 {\tiny (1)}  &  0.51 {\tiny (0)}  & {\bf 0.350 {\tiny (1)}}\\ \cmidrule{2-11}
                        & \multirow{2}{*}{   355M   } &   Standard   &  0.46 {\tiny (0)}  &  0.21 {\tiny (1)}  &  0.31 {\tiny (0)}  &  0.37 {\tiny (2)}  &  0.29 {\tiny (0)}  &  0.42 {\tiny (1)}  &  0.51 {\tiny (1)}  &  0.366 {\tiny (4)}\\
                        &                          &   Coupled   & {\bf 0.47 {\tiny (1)}} &  0.21 {\tiny (1)}  & {\bf 0.32 {\tiny (0)}} &  0.38 {\tiny (1)}  &  0.30 {\tiny (0)}  &  0.42 {\tiny (1)}  &  0.51 {\tiny (0)}  &  0.372 {\tiny (6)}\\ \cmidrule{2-11}
                        & \multirow{2}{*}{   760M   } &   Standard   &  0.49 {\tiny (0)}  &  0.21 {\tiny (1)}  &  0.33 {\tiny (0)}  & {\bf 0.42 {\tiny (0)}} &  0.30 {\tiny (1)}  &  0.41 {\tiny (1)}  &  0.53 {\tiny (0)}  &  0.385 {\tiny (3)}\\
                        &                          &   Coupled   &  0.51 {\tiny (1)}  & {\bf 0.22 {\tiny (1)}} &  0.33 {\tiny (0)}  &  0.41 {\tiny (0)}  &  0.31 {\tiny (1)}  &  0.41 {\tiny (1)}  & {\bf 0.54 {\tiny (0)}} & {\bf 0.392 {\tiny (2)}}\\ 
}
\newcommand{\resultsLdownstream}{ \multirow{2}{*}{26B  } & \multirow{2}{*}{  1.3B  } &  Standard  & {\bf 0.55} & {\bf 0.23} & {\bf 0.36} & {\bf 0.45} & {\bf 0.32} &  0.37  & {\bf 0.54} & {\bf 0.402}\\
                        &                          &  Coupled  &  0.52  &  0.23  & {\bf 0.36} &  0.43  &  0.32  & {\bf 0.38} &  0.53  &  0.396\\ \midrule
 \multirow{2}{*}{52B  } & \multirow{2}{*}{  2.6B  } &  Standard  & {\bf 0.61} & {\bf 0.27} & {\bf 0.43} & {\bf 0.55} & {\bf 0.35} & {\bf 0.38} & {\bf 0.58} & {\bf 0.451}\\
                        &                          &  Coupled  &  0.60  &  0.25  &  0.42  &  0.54  &  0.34  &  0.37  &  0.56  &  0.441\\ \midrule
 \multirow{2}{*}{105B  } & \multirow{2}{*}{  1.3B  } &  Standard  &  0.58  & {\bf 0.27} & {\bf 0.42} & {\bf 0.55} & {\bf 0.35} &  0.38  &  0.57  &  0.446\\
                        &                          &  Coupled  & {\bf 0.60} &  0.26  &  0.42  &  0.54  &  0.35  & {\bf 0.39} & {\bf 0.57} & {\bf 0.447}\\ \midrule
 \multirow{2}{*}{210B  } & \multirow{2}{*}{  2.6B  } &  Standard  &  0.65  &  0.29  & {\bf 0.48} & {\bf 0.63} & {\bf 0.37} &  0.37  & {\bf 0.63} &  0.490\\
                        &                          &  Coupled  & {\bf 0.67} & {\bf 0.32} & {\bf 0.48} &  0.61  &  0.37  & {\bf 0.39} &  0.61  & {\bf 0.492}\\ 
}
\newcommand{\resultsAblationsScaleSmalldownstream}{-5   &  0.43 {\tiny (0)}  &  0.20 {\tiny (0)}  &  0.28 {\tiny (0)}  &  0.30 {\tiny (1)}  &  0.29 {\tiny (1)}  &  0.43 {\tiny (1)}  &  0.51 {\tiny (0)}  &  0.349 {\tiny (2)}\\
-4   &  0.43 {\tiny (1)}  &  0.19 {\tiny (1)}  &  0.28 {\tiny (0)}  &  0.31 {\tiny (1)}  &  0.28 {\tiny (0)}  &  0.44 {\tiny (1)}  &  0.51 {\tiny (1)}  &  0.348 {\tiny (5)}\\
-3   &  0.43 {\tiny (0)}  &  0.20 {\tiny (1)}  &  0.28 {\tiny (0)}  &  0.31 {\tiny (0)}  &  0.28 {\tiny (0)}  &  0.44 {\tiny (2)}  &  0.52 {\tiny (1)}  &  0.352 {\tiny (5)}\\
-2   &  0.43 {\tiny (1)}  &  0.19 {\tiny (1)}  &  0.29 {\tiny (0)}  &  0.31 {\tiny (1)}  &  0.29 {\tiny (1)}  &  0.44 {\tiny (0)}  &  0.52 {\tiny (1)}  &  0.352 {\tiny (1)}\\
-1   &  0.43 {\tiny (0)}  &  0.19 {\tiny (1)}  &  0.29 {\tiny (0)}  &  0.31 {\tiny (1)}  &  0.28 {\tiny (1)}  &  0.43 {\tiny (1)}  &  0.50 {\tiny (1)}  &  0.348 {\tiny (3)}\\ \cmidrule{1-9}
0   &  0.44 {\tiny (1)}  &  0.19 {\tiny (1)}  &  0.29 {\tiny (0)}  &  0.31 {\tiny (0)}  &  0.29 {\tiny (1)}  &  0.43 {\tiny (1)}  &  0.51 {\tiny (0)}  &  0.350 {\tiny (1)}\\ \cmidrule{1-9}
1   &  0.43 {\tiny (1)}  &  0.20 {\tiny (1)}  &  0.29 {\tiny (0)}  &  0.32 {\tiny (1)}  &  0.29 {\tiny (0)}  &  0.43 {\tiny (2)}  &  0.51 {\tiny (0)}  &  0.351 {\tiny (4)}\\
2   &  0.44 {\tiny (0)}  &  0.19 {\tiny (0)}  &  0.29 {\tiny (0)}  &  0.31 {\tiny (1)}  &  0.29 {\tiny (1)}  &  0.43 {\tiny (1)}  &  0.52 {\tiny (0)}  &  0.353 {\tiny (2)}\\
3   &  0.45 {\tiny (1)}  &  0.19 {\tiny (0)}  &  0.29 {\tiny (0)}  &  0.31 {\tiny (0)}  &  0.29 {\tiny (1)}  &  0.43 {\tiny (0)}  &  0.51 {\tiny (0)}  &  0.352 {\tiny (0)}\\
4   &  0.43 {\tiny (0)}  &  0.19 {\tiny (1)}  &  0.28 {\tiny (0)}  &  0.31 {\tiny (0)}  &  0.30 {\tiny (1)}  &  0.43 {\tiny (1)}  &  0.52 {\tiny (1)}  &  0.352 {\tiny (1)}\\
5   &  0.44 {\tiny (1)}  &  0.19 {\tiny (1)}  &  0.29 {\tiny (0)}  &  0.30 {\tiny (0)}  &  0.28 {\tiny (0)}  &  0.43 {\tiny (1)}  &  0.51 {\tiny (0)}  &  0.349 {\tiny (4)}\\}
\newcommand{\resultsAblationsSGDAlldownstream}{ \multirow{2}{*}{5B   } & \multirow{2}{*}{   125M   } &   SGD (300)   &  0.40 {\tiny (1)}  &  0.18 {\tiny (1)}  &  0.27 {\tiny (0)}  &  0.26 {\tiny (0)}  &  0.28 {\tiny (0)}  &  0.44 {\tiny (1)}  &  0.50 {\tiny (2)}  &  0.333 {\tiny (3)}\\
                        &                          &   Coupled Adam   &  0.40 {\tiny (0)}  &  0.18 {\tiny (1)}  &  0.28 {\tiny (0)}  & {\bf 0.27 {\tiny (0)}} &  0.28 {\tiny (0)}  &  0.45 {\tiny (1)}  &  0.52 {\tiny (1)}  & {\bf 0.339 {\tiny (2)}}\\ \midrule
 \multirow{2}{*}{10B   } & \multirow{2}{*}{   125M   } &   SGD (300)   &  0.41 {\tiny (0)}  &  0.19 {\tiny (1)}  &  0.28 {\tiny (0)}  &  0.28 {\tiny (1)}  &  0.28 {\tiny (1)}  &  0.44 {\tiny (1)}  &  0.50 {\tiny (1)}  &  0.341 {\tiny (4)}\\
                        &                          &   Coupled Adam   &  0.42 {\tiny (1)}  &  0.19 {\tiny (1)}  & {\bf 0.28 {\tiny (0)}} &  0.29 {\tiny (0)}  &  0.28 {\tiny (1)}  &  0.44 {\tiny (0)}  &  0.50 {\tiny (1)}  &  0.343 {\tiny (1)}\\ \midrule
 \multirow{2}{*}{20B   } & \multirow{2}{*}{   125M   } &   SGD (400)   &  0.42 {\tiny (1)}  & {\bf 0.20 {\tiny (0)}} &  0.28 {\tiny (0)}  &  0.30 {\tiny (1)}  &  0.28 {\tiny (1)}  &  0.43 {\tiny (0)}  &  0.51 {\tiny (2)}  &  0.346 {\tiny (5)}\\
                        &                          &   Coupled Adam   &  0.44 {\tiny (1)}  &  0.19 {\tiny (1)}  &  0.29 {\tiny (0)}  &  0.31 {\tiny (0)}  &  0.29 {\tiny (1)}  &  0.43 {\tiny (1)}  &  0.51 {\tiny (0)}  &  0.350 {\tiny (1)}\\ 
}

\section{Additional Results}

\subsection{Scaled Coupled Adam}
\label{app:additional_results_scale}

Tab.~\ref{tab:results_ablations_scale} of 
Sec.~\ref{sec:ablation_different_learning_rate} shows the results of varying the scaling exponent $n$ (see Eq.~(\ref{eq:optimizer_update_second_moment_avg_scaled})) for $D = 20\B$. The dependency of the loss is visualized in Fig.~\ref{fig:ablation_scale}.
Here, in Fig.~\ref{fig:ablation_scale_complete}, we extend the visualization of the results to $D \in \{ 5\B, 10\B, 20\B \}$ and the other evaluation metrics.
\begin{figure*}[ht]
    \centering
    \includegraphics[scale=0.5]{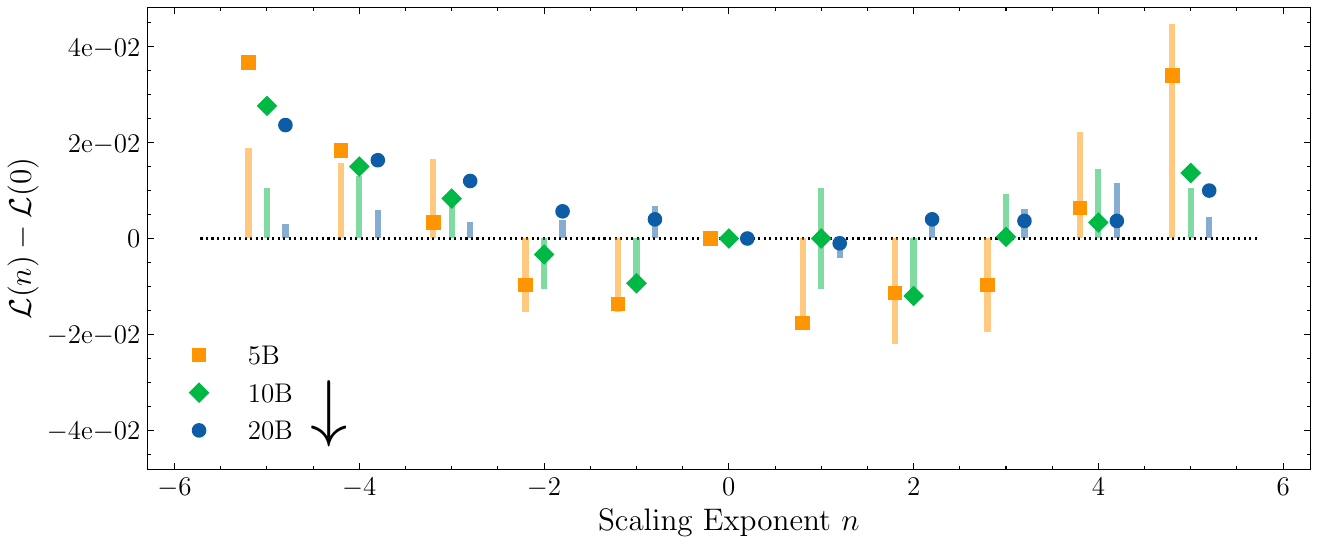}
    \includegraphics[scale=0.5]{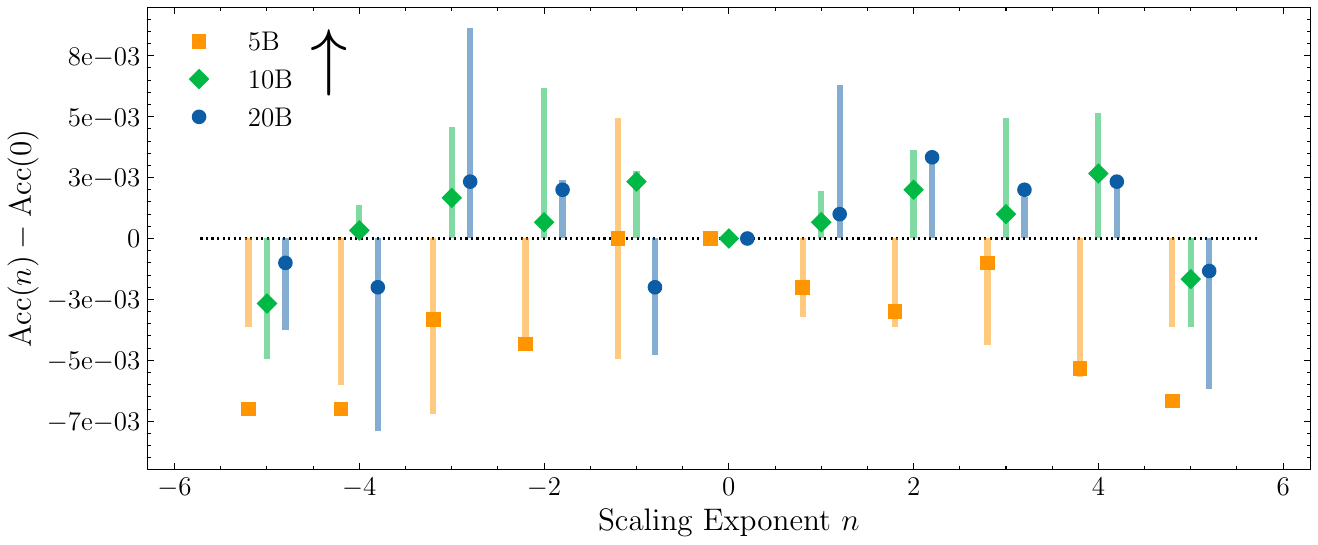}
    \includegraphics[scale=0.5]{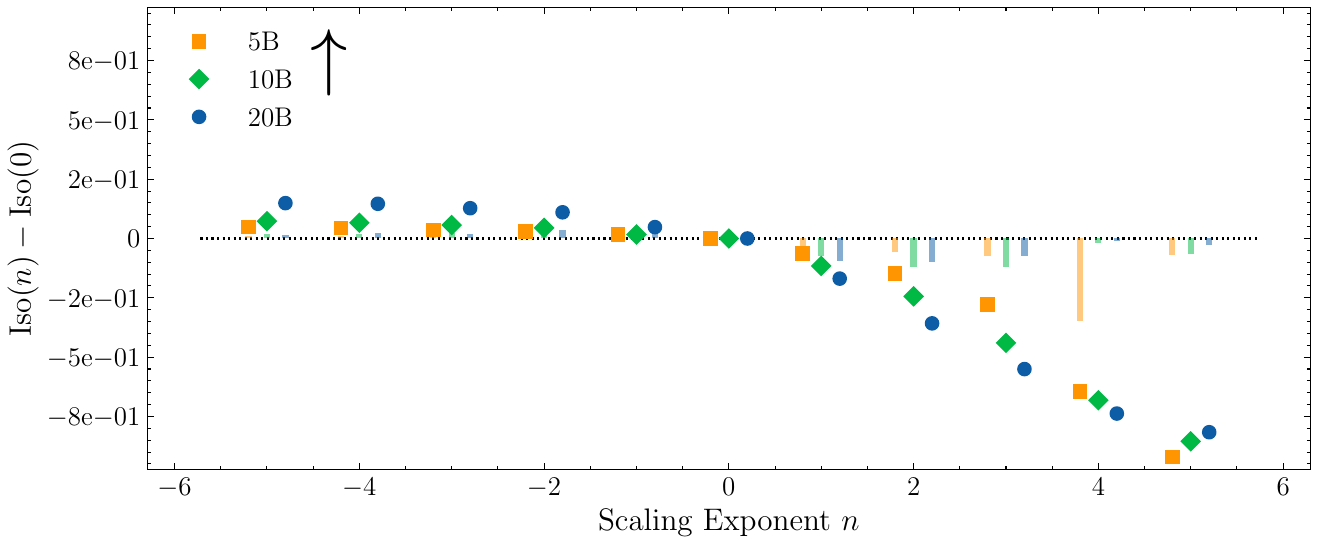}
    \includegraphics[scale=0.5]{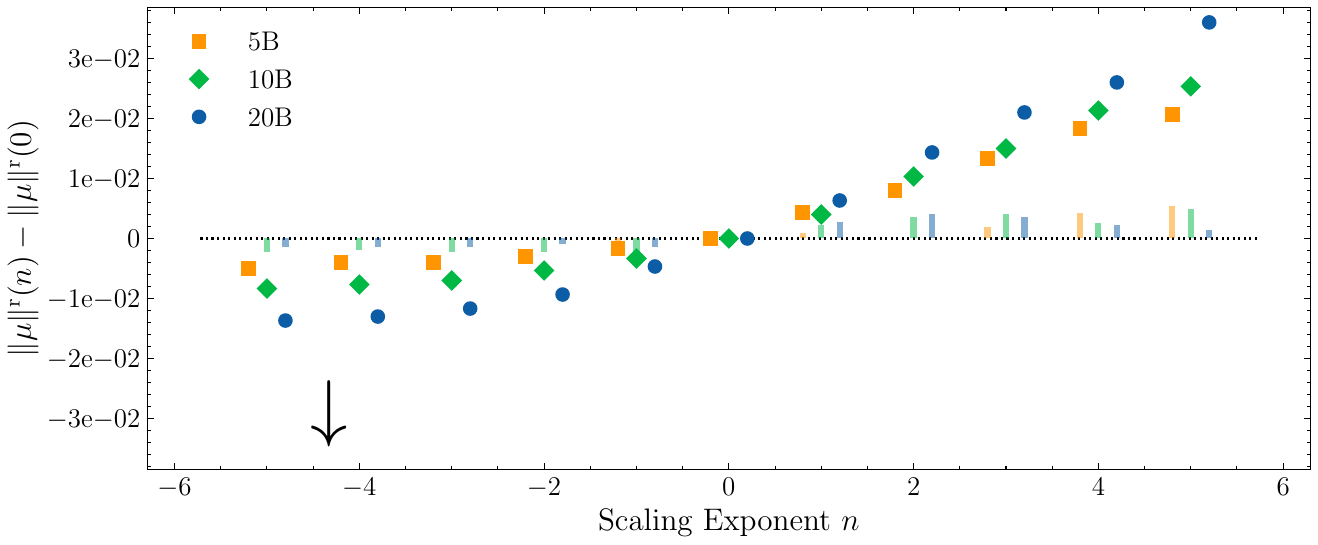}
    \includegraphics[scale=0.5]{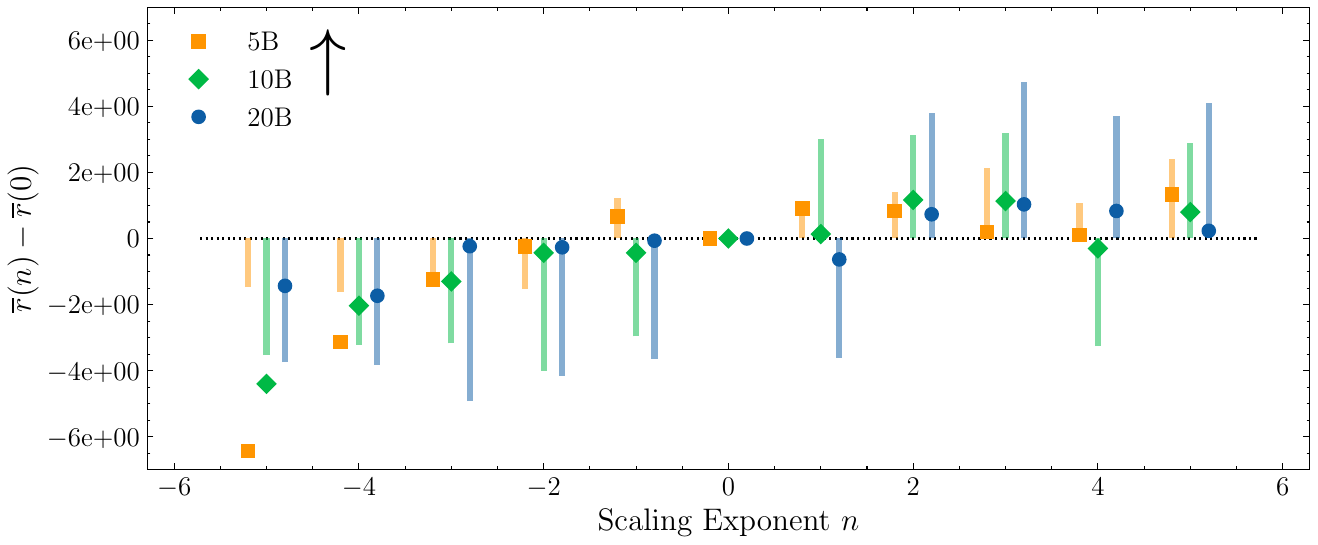}
    \caption{Dependency of different metrics on the scaling exponent $n$, see Eq.~(\ref{eq:optimizer_update_second_moment_avg_scaled}). From top to bottom: loss (upstream performance), average accuracy (downstream performance), isotropy, mean embedding norm ratio and $\rcos$. Each plot shows the difference to the respective metric obtained for $n = 0$. The arrows indicate whether larger ($\uparrow$) or smaller ($\downarrow$) values are desirable.}
    \label{fig:ablation_scale_complete}
\end{figure*}

\subsection{SGD}
\label{app:additional_results_sgd}

In Tab.~\ref{tab:results_ablations_sgd_all} of Sec.~\ref{sec:ablation_sgd}, we showed results for SGD using the best hyperparameter $f$. 
Detailed results of the corresponding hyperparameter searches can be found in Tab.~\ref{tab:results_ablations_sgd}.
\begin{table*}[ht]
\centering
\scriptsize
\begin{tabular}{ccc|rrrrrrrr}
\toprule
$D$ & $N$ & Optimizer & $\Loss$ ($\downarrow$) & $\Acc$ ($\uparrow$) & $\ISO$ ($\uparrow$) & $\munorm$ ($\downarrow$) & $\munormrel$ ($\downarrow$) & $\rcos$ ($\uparrow$) & $\rho$ ($\uparrow$) & $\kappa$ ($\uparrow$) \\ 
\midrule
 \multirow{8}{*}{5B   } & \multirow{8}{*}{   125M   } &   SGD (100)   &  3.23 {\tiny (0)}  &  0.325 {\tiny (1)}  &  1.00 {\tiny (0)}  & {\bf 0.00 {\tiny (0)}} & {\bf 0.01 {\tiny (0)}} &  33 {\tiny (1)}  &  59 {\tiny (1)}  & {\bf 25.0 {\tiny (1.1)}}\\
                        &                        &   SGD (200)   &  3.19 {\tiny (1)}  &  0.327 {\tiny (2)}  &  1.00 {\tiny (0)}  &  0.00 {\tiny (0)}  &  0.01 {\tiny (0)}  &  41 {\tiny (1)}  &  66 {\tiny (1)}  &  19.2 {\tiny (9)}\\
                        &                        &   SGD (300)   &  3.17 {\tiny (0)}  &  0.333 {\tiny (3)}  &  0.99 {\tiny (0)}  &  0.00 {\tiny (0)}  &  0.01 {\tiny (0)}  &  45 {\tiny (1)}  &  71 {\tiny (1)}  &  15.5 {\tiny (4)}\\
                        &                        &   SGD (400)   &  3.18 {\tiny (2)}  &  0.332 {\tiny (3)}  &  0.99 {\tiny (0)}  &  0.01 {\tiny (0)}  &  0.01 {\tiny (0)}  &  48 {\tiny (3)}  &  75 {\tiny (1)}  &  14.2 {\tiny (5)}\\
                        &                        &   SGD (500)   &  3.19 {\tiny (2)}  &  0.330 {\tiny (6)}  &  0.99 {\tiny (0)}  &  0.01 {\tiny (0)}  &  0.01 {\tiny (0)}  &  49 {\tiny (1)}  &  77 {\tiny (1)}  &  13.2 {\tiny (2)}\\
                        &                        &   SGD (600)   &  3.24 {\tiny (2)}  &  0.326 {\tiny (3)}  &  0.99 {\tiny (0)}  &  0.01 {\tiny (0)}  &  0.01 {\tiny (0)}  &  48 {\tiny (2)}  &  79 {\tiny (1)}  &  11.6 {\tiny (2)}\\ \cmidrule{3-11}
                        &                        &   Standard Adam   &  3.14 {\tiny (0)}  &  0.340 {\tiny (2)}  &  0.31 {\tiny (2)}  &  1.10 {\tiny (6)}  &  0.67 {\tiny (3)}  &  15 {\tiny (3)}  &  -54 {\tiny (3)}  &  0.6 {\tiny (1)}\\
                        &                        &   Coupled Adam   & {\bf 3.12 {\tiny (1)}} &  0.339 {\tiny (2)}  &  0.94 {\tiny (1)}  &  0.02 {\tiny (0)}  &  0.01 {\tiny (0)}  & {\bf 55 {\tiny (0)}} & {\bf 87 {\tiny (1)}} &  4.8 {\tiny (2)}\\

\bottomrule 
\\
\\
\end{tabular}
\begin{tabular}{ccc|rrrrrrrr}
\toprule
$D$ & $N$ & Optimizer & $\Loss$ ($\downarrow$) & $\Acc$ ($\uparrow$) & $\ISO$ ($\uparrow$) & $\munorm$ ($\downarrow$) & $\munormrel$ ($\downarrow$) & $\rcos$ ($\uparrow$) & $\rho$ ($\uparrow$) & $\kappa$ ($\uparrow$) \\ 
\midrule
 \multirow{8}{*}{10B   } & \multirow{8}{*}{   125M   } &   SGD (100)   &  3.13 {\tiny (1)}  &  0.337 {\tiny (4)}  &  1.00 {\tiny (0)}  & {\bf 0.00 {\tiny (0)}} & {\bf 0.01 {\tiny (0)}} &  41 {\tiny (1)}  &  58 {\tiny (4)}  & {\bf 22.2 {\tiny (7)}}\\
                        &                        &   SGD (200)   &  3.09 {\tiny (0)}  &  0.339 {\tiny (4)}  &  0.99 {\tiny (0)}  &  0.01 {\tiny (0)}  &  0.01 {\tiny (0)}  &  48 {\tiny (1)}  &  65 {\tiny (4)}  &  16.1 {\tiny (3)}\\
                        &                        &   SGD (300)   &  3.07 {\tiny (1)}  &  0.341 {\tiny (4)}  &  0.99 {\tiny (0)}  &  0.01 {\tiny (0)}  &  0.01 {\tiny (0)}  &  49 {\tiny (1)}  &  71 {\tiny (4)}  &  12.2 {\tiny (2)}\\
                        &                        &   SGD (400)   &  3.07 {\tiny (0)}  &  0.340 {\tiny (3)}  &  0.99 {\tiny (1)}  &  0.01 {\tiny (0)}  &  0.01 {\tiny (0)}  &  51 {\tiny (2)}  &  76 {\tiny (3)}  &  10.7 {\tiny (8)}\\
                        &                        &   SGD (500)   &  3.09 {\tiny (0)}  &  0.338 {\tiny (3)}  &  0.99 {\tiny (0)}  &  0.01 {\tiny (0)}  &  0.01 {\tiny (0)}  &  52 {\tiny (1)}  &  79 {\tiny (3)}  &  10.3 {\tiny (2)}\\
                        &                        &   SGD (600)   &  3.11 {\tiny (1)}  &  0.341 {\tiny (4)}  &  0.98 {\tiny (1)}  &  0.01 {\tiny (0)}  &  0.01 {\tiny (0)}  &  53 {\tiny (0)}  &  81 {\tiny (1)}  &  9.4 {\tiny (7)}\\ \cmidrule{3-11}
                        &                        &   Standard Adam   &  3.07 {\tiny (0)}  &  0.343 {\tiny (3)}  &  0.21 {\tiny (3)}  &  1.58 {\tiny (5)}  &  0.75 {\tiny (0)}  &  9 {\tiny (2)}  &  -64 {\tiny (5)}  &  0.4 {\tiny (0)}\\
                        &                        &   Coupled Adam   & {\bf 3.03 {\tiny (0)}} &  0.343 {\tiny (1)}  &  0.91 {\tiny (1)}  &  0.05 {\tiny (0)}  &  0.02 {\tiny (0)}  & {\bf 57 {\tiny (2)}} &  82 {\tiny (0)}  &  3.6 {\tiny (8)}\\

\bottomrule 
\\
\\
\end{tabular}
\begin{tabular}{ccc|rrrrrrrr}
\toprule
$D$ & $N$ & Optimizer & $\Loss$ ($\downarrow$) & $\Acc$ ($\uparrow$) & $\ISO$ ($\uparrow$) & $\munorm$ ($\downarrow$) & $\munormrel$ ($\downarrow$) & $\rcos$ ($\uparrow$) & $\rho$ ($\uparrow$) & $\kappa$ ($\uparrow$) \\ 
\midrule
 \multirow{8}{*}{20B   } & \multirow{8}{*}{   125M   } &   SGD (100)   &  3.05 {\tiny (1)}  &  0.343 {\tiny (2)}  &  0.99 {\tiny (0)}  & {\bf 0.01 {\tiny (0)}} & {\bf 0.01 {\tiny (0)}} &  50 {\tiny (1)}  &  57 {\tiny (7)}  & {\bf 18.4 {\tiny (4)}}\\
                        &                        &   SGD (200)   &  3.02 {\tiny (0)}  &  0.345 {\tiny (1)}  &  0.99 {\tiny (0)}  &  0.01 {\tiny (0)}  &  0.01 {\tiny (0)}  &  52 {\tiny (1)}  &  64 {\tiny (7)}  &  12.7 {\tiny (7)}\\
                        &                        &   SGD (300)   &  3.01 {\tiny (1)}  &  0.350 {\tiny (1)}  &  0.98 {\tiny (1)}  &  0.01 {\tiny (0)}  &  0.02 {\tiny (0)}  &  53 {\tiny (2)}  &  70 {\tiny (6)}  &  9.0 {\tiny (2)}\\
                        &                        &   SGD (400)   &  3.00 {\tiny (0)}  &  0.346 {\tiny (5)}  &  0.98 {\tiny (1)}  &  0.01 {\tiny (0)}  &  0.02 {\tiny (0)}  &  54 {\tiny (0)}  &  76 {\tiny (5)}  &  7.4 {\tiny (1.1)}\\
                        &                        &   SGD (500)   &  3.01 {\tiny (1)}  &  0.348 {\tiny (4)}  &  0.98 {\tiny (1)}  &  0.02 {\tiny (0)}  &  0.02 {\tiny (0)}  &  55 {\tiny (1)}  &  78 {\tiny (5)}  &  7.5 {\tiny (1.6)}\\
                        &                        &   SGD (600)   &  3.04 {\tiny (3)}  &  0.348 {\tiny (5)}  &  0.95 {\tiny (3)}  &  0.02 {\tiny (0)}  &  0.02 {\tiny (0)}  &  52 {\tiny (5)}  &  81 {\tiny (5)}  &  6.8 {\tiny (2.8)}\\ \cmidrule{3-11}
                        &                        &   Standard Adam   &  3.03 {\tiny (0)}  &  0.346 {\tiny (1)}  &  0.10 {\tiny (3)}  &  2.14 {\tiny (7)}  &  0.82 {\tiny (2)}  &  5 {\tiny (1)}  &  -66 {\tiny (2)}  &  0.3 {\tiny (0)}\\
                        &                        &   Coupled Adam   & {\bf 2.97 {\tiny (0)}} &  0.350 {\tiny (1)}  &  0.83 {\tiny (1)}  &  0.11 {\tiny (0)}  &  0.03 {\tiny (0)}  &  57 {\tiny (2)}  &  77 {\tiny (0)}  &  1.7 {\tiny (5)}\\

\bottomrule 
\end{tabular}
\caption{Results of our experiments with SGD. Values are highlighted in bold if they are significantly better than all the other values in the same column.}
\label{tab:results_ablations_sgd}
\end{table*}

\subsection{Individual Downstream Task Performance}
\label{app:additional_results_downstream}

In Tab.~\ref{tab:results_S_downstream}-\ref{tab:results_ablations_sgd_all_downstream}, we list the individual downstream task performance for all our experiments (Sec.~\ref{sec:experiments}-\ref{sec:ablations}).

\clearpage

\begin{table*}
\centering
\scriptsize
\begin{tabular}{ccc|rrrrrrr|r}
\toprule
$D$ & $N$ & Adam & ARC easy & ARC challenge & HellaSwag & LAMBADA & RACE & TruthfulQA & WinoGrande & $\Acc$ ($\uparrow$) \\ 
\midrule
 \multirow{6}{*}{5B   } & \multirow{2}{*}{   125M   } &   Standard   &  0.41 {\tiny (0)}  &  0.19 {\tiny (1)}  &  0.27 {\tiny (0)}  &  0.26 {\tiny (0)}  &  0.28 {\tiny (1)}  &  0.45 {\tiny (0)}  &  0.52 {\tiny (1)}  &  0.340 {\tiny (2)}\\
                        &                          &   Coupled   &  0.40 {\tiny (0)}  &  0.18 {\tiny (1)}  &  0.28 {\tiny (0)}  & {\bf 0.27 {\tiny (0)}} &  0.28 {\tiny (0)}  &  0.45 {\tiny (1)}  &  0.52 {\tiny (1)}  &  0.339 {\tiny (2)}\\ \cmidrule{2-11}
                        & \multirow{2}{*}{   355M   } &   Standard   &  0.43 {\tiny (1)}  &  0.19 {\tiny (1)}  &  0.29 {\tiny (0)}  &  0.31 {\tiny (0)}  &  0.29 {\tiny (0)}  &  0.43 {\tiny (0)}  &  0.52 {\tiny (1)}  &  0.352 {\tiny (3)}\\
                        &                          &   Coupled   &  0.43 {\tiny (2)}  &  0.19 {\tiny (1)}  &  0.29 {\tiny (0)}  &  0.31 {\tiny (1)}  &  0.29 {\tiny (0)}  &  0.43 {\tiny (1)}  &  0.51 {\tiny (0)}  &  0.350 {\tiny (4)}\\ \cmidrule{2-11}
                        & \multirow{2}{*}{   760M   } &   Standard   &  0.46 {\tiny (1)}  &  0.20 {\tiny (1)}  &  0.30 {\tiny (0)}  &  0.34 {\tiny (1)}  &  0.29 {\tiny (1)}  &  0.43 {\tiny (1)}  &  0.50 {\tiny (1)}  &  0.360 {\tiny (3)}\\
                        &                          &   Coupled   &  0.45 {\tiny (0)}  &  0.20 {\tiny (1)}  &  0.30 {\tiny (0)}  &  0.34 {\tiny (1)}  &  0.29 {\tiny (1)}  &  0.42 {\tiny (1)}  &  0.50 {\tiny (0)}  &  0.357 {\tiny (3)}\\ \midrule
 \multirow{6}{*}{10B   } & \multirow{2}{*}{   125M   } &   Standard   &  0.42 {\tiny (1)}  &  0.19 {\tiny (0)}  &  0.28 {\tiny (0)}  &  0.29 {\tiny (1)}  &  0.28 {\tiny (1)}  &  0.44 {\tiny (1)}  &  0.51 {\tiny (1)}  &  0.343 {\tiny (3)}\\
                        &                          &   Coupled   &  0.42 {\tiny (1)}  &  0.19 {\tiny (1)}  & {\bf 0.28 {\tiny (0)}} &  0.29 {\tiny (0)}  &  0.28 {\tiny (1)}  &  0.44 {\tiny (0)}  &  0.50 {\tiny (1)}  &  0.343 {\tiny (1)}\\ \cmidrule{2-11}
                        & \multirow{2}{*}{   355M   } &   Standard   &  0.45 {\tiny (1)}  &  0.19 {\tiny (1)}  &  0.30 {\tiny (0)}  &  0.35 {\tiny (0)}  &  0.29 {\tiny (1)}  &  0.41 {\tiny (0)}  &  0.51 {\tiny (1)}  &  0.359 {\tiny (2)}\\
                        &                          &   Coupled   &  0.46 {\tiny (1)}  &  0.19 {\tiny (1)}  &  0.30 {\tiny (0)}  &  0.35 {\tiny (1)}  &  0.30 {\tiny (1)}  & {\bf 0.42 {\tiny (0)}} &  0.52 {\tiny (1)}  & {\bf 0.365 {\tiny (2)}}\\ \cmidrule{2-11}
                        & \multirow{2}{*}{   760M   } &   Standard   &  0.47 {\tiny (0)}  &  0.21 {\tiny (0)}  &  0.32 {\tiny (0)}  & {\bf 0.39 {\tiny (1)}} &  0.30 {\tiny (1)}  &  0.41 {\tiny (0)}  &  0.53 {\tiny (1)}  &  0.375 {\tiny (2)}\\
                        &                          &   Coupled   &  0.48 {\tiny (1)}  &  0.21 {\tiny (1)}  &  0.32 {\tiny (0)}  &  0.38 {\tiny (0)}  &  0.30 {\tiny (1)}  &  0.41 {\tiny (1)}  &  0.51 {\tiny (1)}  &  0.372 {\tiny (3)}\\ \midrule
 \multirow{6}{*}{20B   } & \multirow{2}{*}{   125M   } &   Standard   &  0.43 {\tiny (0)}  &  0.18 {\tiny (1)}  &  0.28 {\tiny (0)}  &  0.29 {\tiny (1)}  &  0.28 {\tiny (1)}  &  0.44 {\tiny (1)}  &  0.51 {\tiny (1)}  &  0.346 {\tiny (1)}\\
                        &                          &   Coupled   &  0.44 {\tiny (1)}  &  0.19 {\tiny (1)}  &  0.29 {\tiny (0)}  & {\bf 0.31 {\tiny (0)}} &  0.29 {\tiny (1)}  &  0.43 {\tiny (1)}  &  0.51 {\tiny (0)}  & {\bf 0.350 {\tiny (1)}}\\ \cmidrule{2-11}
                        & \multirow{2}{*}{   355M   } &   Standard   &  0.46 {\tiny (0)}  &  0.21 {\tiny (1)}  &  0.31 {\tiny (0)}  &  0.37 {\tiny (2)}  &  0.29 {\tiny (0)}  &  0.42 {\tiny (1)}  &  0.51 {\tiny (1)}  &  0.366 {\tiny (4)}\\
                        &                          &   Coupled   & {\bf 0.47 {\tiny (1)}} &  0.21 {\tiny (1)}  & {\bf 0.32 {\tiny (0)}} &  0.38 {\tiny (1)}  &  0.30 {\tiny (0)}  &  0.42 {\tiny (1)}  &  0.51 {\tiny (0)}  &  0.372 {\tiny (6)}\\ \cmidrule{2-11}
                        & \multirow{2}{*}{   760M   } &   Standard   &  0.49 {\tiny (0)}  &  0.21 {\tiny (1)}  &  0.33 {\tiny (0)}  & {\bf 0.42 {\tiny (0)}} &  0.30 {\tiny (1)}  &  0.41 {\tiny (1)}  &  0.53 {\tiny (0)}  &  0.385 {\tiny (3)}\\
                        &                          &   Coupled   &  0.51 {\tiny (1)}  & {\bf 0.22 {\tiny (1)}} &  0.33 {\tiny (0)}  &  0.41 {\tiny (0)}  &  0.31 {\tiny (1)}  &  0.41 {\tiny (1)}  & {\bf 0.54 {\tiny (0)}} & {\bf 0.392 {\tiny (2)}}\\ 

\bottomrule 
\end{tabular}
\caption{Detailed downstream task results of our small-scale experiments from Sec.~\ref{sec:experiments_S} and \ref{sec:results_S}. Compare to Tab.~\ref{tab:results_S}.}
\label{tab:results_S_downstream}
\end{table*}

\begin{table*}
\centering
\scriptsize
\begin{tabular}{lll|rrrrrrr|r}
\toprule
$D$ & $N$ & Adam & ARC easy & ARC challenge & HellaSwag & LAMBADA & RACE & TruthfulQA & WinoGrande & $\Acc$ ($\uparrow$) \\
\midrule
 \multirow{2}{*}{26B  } & \multirow{2}{*}{  1.3B  } &  Standard  & {\bf 0.55} & {\bf 0.23} & {\bf 0.36} & {\bf 0.45} & {\bf 0.32} &  0.37  & {\bf 0.54} & {\bf 0.402}\\
                        &                          &  Coupled  &  0.52  &  0.23  & {\bf 0.36} &  0.43  &  0.32  & {\bf 0.38} &  0.53  &  0.396\\ \midrule
 \multirow{2}{*}{52B  } & \multirow{2}{*}{  2.6B  } &  Standard  & {\bf 0.61} & {\bf 0.27} & {\bf 0.43} & {\bf 0.55} & {\bf 0.35} & {\bf 0.38} & {\bf 0.58} & {\bf 0.451}\\
                        &                          &  Coupled  &  0.60  &  0.25  &  0.42  &  0.54  &  0.34  &  0.37  &  0.56  &  0.441\\ \midrule
 \multirow{2}{*}{105B  } & \multirow{2}{*}{  1.3B  } &  Standard  &  0.58  & {\bf 0.27} & {\bf 0.42} & {\bf 0.55} & {\bf 0.35} &  0.38  &  0.57  &  0.446\\
                        &                          &  Coupled  & {\bf 0.60} &  0.26  &  0.42  &  0.54  &  0.35  & {\bf 0.39} & {\bf 0.57} & {\bf 0.447}\\ \midrule
 \multirow{2}{*}{210B  } & \multirow{2}{*}{  2.6B  } &  Standard  &  0.65  &  0.29  & {\bf 0.48} & {\bf 0.63} & {\bf 0.37} &  0.37  & {\bf 0.63} &  0.490\\
                        &                          &  Coupled  & {\bf 0.67} & {\bf 0.32} & {\bf 0.48} &  0.61  &  0.37  & {\bf 0.39} &  0.61  & {\bf 0.492}\\ 

\bottomrule 
\end{tabular}
\caption{Detailed downstream task results of our large-scale experiments from Sec.~\ref{sec:experiments_L} and \ref{sec:results_L}. Compare to Tab.~\ref{tab:results_L}.}
\label{tab:results_L_downstream}
\end{table*}

\begin{table*}
\centering
\scriptsize
\begin{tabular}{c|rrrrrrr|r}
\toprule
$n$ & ARC easy & ARC challenge & HellaSwag & LAMBADA & RACE & TruthfulQA & WinoGrande & $\Acc$ ($\uparrow$) \\
\midrule
-5   &  0.43 {\tiny (0)}  &  0.20 {\tiny (0)}  &  0.28 {\tiny (0)}  &  0.30 {\tiny (1)}  &  0.29 {\tiny (1)}  &  0.43 {\tiny (1)}  &  0.51 {\tiny (0)}  &  0.349 {\tiny (2)}\\
-4   &  0.43 {\tiny (1)}  &  0.19 {\tiny (1)}  &  0.28 {\tiny (0)}  &  0.31 {\tiny (1)}  &  0.28 {\tiny (0)}  &  0.44 {\tiny (1)}  &  0.51 {\tiny (1)}  &  0.348 {\tiny (5)}\\
-3   &  0.43 {\tiny (0)}  &  0.20 {\tiny (1)}  &  0.28 {\tiny (0)}  &  0.31 {\tiny (0)}  &  0.28 {\tiny (0)}  &  0.44 {\tiny (2)}  &  0.52 {\tiny (1)}  &  0.352 {\tiny (5)}\\
-2   &  0.43 {\tiny (1)}  &  0.19 {\tiny (1)}  &  0.29 {\tiny (0)}  &  0.31 {\tiny (1)}  &  0.29 {\tiny (1)}  &  0.44 {\tiny (0)}  &  0.52 {\tiny (1)}  &  0.352 {\tiny (1)}\\
-1   &  0.43 {\tiny (0)}  &  0.19 {\tiny (1)}  &  0.29 {\tiny (0)}  &  0.31 {\tiny (1)}  &  0.28 {\tiny (1)}  &  0.43 {\tiny (1)}  &  0.50 {\tiny (1)}  &  0.348 {\tiny (3)}\\ \cmidrule{1-9}
0   &  0.44 {\tiny (1)}  &  0.19 {\tiny (1)}  &  0.29 {\tiny (0)}  &  0.31 {\tiny (0)}  &  0.29 {\tiny (1)}  &  0.43 {\tiny (1)}  &  0.51 {\tiny (0)}  &  0.350 {\tiny (1)}\\ \cmidrule{1-9}
1   &  0.43 {\tiny (1)}  &  0.20 {\tiny (1)}  &  0.29 {\tiny (0)}  &  0.32 {\tiny (1)}  &  0.29 {\tiny (0)}  &  0.43 {\tiny (2)}  &  0.51 {\tiny (0)}  &  0.351 {\tiny (4)}\\
2   &  0.44 {\tiny (0)}  &  0.19 {\tiny (0)}  &  0.29 {\tiny (0)}  &  0.31 {\tiny (1)}  &  0.29 {\tiny (1)}  &  0.43 {\tiny (1)}  &  0.52 {\tiny (0)}  &  0.353 {\tiny (2)}\\
3   &  0.45 {\tiny (1)}  &  0.19 {\tiny (0)}  &  0.29 {\tiny (0)}  &  0.31 {\tiny (0)}  &  0.29 {\tiny (1)}  &  0.43 {\tiny (0)}  &  0.51 {\tiny (0)}  &  0.352 {\tiny (0)}\\
4   &  0.43 {\tiny (0)}  &  0.19 {\tiny (1)}  &  0.28 {\tiny (0)}  &  0.31 {\tiny (0)}  &  0.30 {\tiny (1)}  &  0.43 {\tiny (1)}  &  0.52 {\tiny (1)}  &  0.352 {\tiny (1)}\\
5   &  0.44 {\tiny (1)}  &  0.19 {\tiny (1)}  &  0.29 {\tiny (0)}  &  0.30 {\tiny (0)}  &  0.28 {\tiny (0)}  &  0.43 {\tiny (1)}  &  0.51 {\tiny (0)}  &  0.349 {\tiny (4)}\\
\bottomrule 
\end{tabular}
\caption{Detailed downstream task results of our ablations on Scaled Coupled Adam from Sec.~\ref{sec:ablation_different_learning_rate}. Compare to Tab.~\ref{tab:results_ablations_scale}.}.
\label{tab:results_ablations_scale_downstream}
\end{table*}

\begin{table*}
\centering
\scriptsize
\begin{tabular}{ccc|rrrrrrr|r}
\toprule
$D$ & $N$ & Optimizer & ARC easy & ARC challenge & HellaSwag & LAMBADA & RACE & TruthfulQA & WinoGrande & $\Acc$ ($\uparrow$) \\ 
\midrule
 \multirow{2}{*}{5B   } & \multirow{2}{*}{   125M   } &   SGD (300)   &  0.40 {\tiny (1)}  &  0.18 {\tiny (1)}  &  0.27 {\tiny (0)}  &  0.26 {\tiny (0)}  &  0.28 {\tiny (0)}  &  0.44 {\tiny (1)}  &  0.50 {\tiny (2)}  &  0.333 {\tiny (3)}\\
                        &                          &   Coupled Adam   &  0.40 {\tiny (0)}  &  0.18 {\tiny (1)}  &  0.28 {\tiny (0)}  & {\bf 0.27 {\tiny (0)}} &  0.28 {\tiny (0)}  &  0.45 {\tiny (1)}  &  0.52 {\tiny (1)}  & {\bf 0.339 {\tiny (2)}}\\ \midrule
 \multirow{2}{*}{10B   } & \multirow{2}{*}{   125M   } &   SGD (300)   &  0.41 {\tiny (0)}  &  0.19 {\tiny (1)}  &  0.28 {\tiny (0)}  &  0.28 {\tiny (1)}  &  0.28 {\tiny (1)}  &  0.44 {\tiny (1)}  &  0.50 {\tiny (1)}  &  0.341 {\tiny (4)}\\
                        &                          &   Coupled Adam   &  0.42 {\tiny (1)}  &  0.19 {\tiny (1)}  & {\bf 0.28 {\tiny (0)}} &  0.29 {\tiny (0)}  &  0.28 {\tiny (1)}  &  0.44 {\tiny (0)}  &  0.50 {\tiny (1)}  &  0.343 {\tiny (1)}\\ \midrule
 \multirow{2}{*}{20B   } & \multirow{2}{*}{   125M   } &   SGD (400)   &  0.42 {\tiny (1)}  & {\bf 0.20 {\tiny (0)}} &  0.28 {\tiny (0)}  &  0.30 {\tiny (1)}  &  0.28 {\tiny (1)}  &  0.43 {\tiny (0)}  &  0.51 {\tiny (2)}  &  0.346 {\tiny (5)}\\
                        &                          &   Coupled Adam   &  0.44 {\tiny (1)}  &  0.19 {\tiny (1)}  &  0.29 {\tiny (0)}  &  0.31 {\tiny (0)}  &  0.29 {\tiny (1)}  &  0.43 {\tiny (1)}  &  0.51 {\tiny (0)}  &  0.350 {\tiny (1)}\\ 

\bottomrule 
\end{tabular}
\caption{Detailed downstream task results of our ablations on SGD from Sec.~\ref{sec:ablation_sgd}. Compare to Tab.~\ref{tab:results_ablations_sgd_all}.}
\label{tab:results_ablations_sgd_all_downstream}
\end{table*}

\end{document}